\documentclass[conference]{IEEEtran}
\IEEEoverridecommandlockouts
\usepackage{cite}
\usepackage{amsmath,amssymb,amsfonts}
\usepackage{algorithm}
\usepackage{algorithmic}
\usepackage{graphicx}
\usepackage{textcomp}
\usepackage{xcolor}
\def\BibTeX{{\rm B\kern-.05em{\sc i\kern-.025em b}\kern-.08em
    T\kern-.1667em\lower.7ex\hbox{E}\kern-.125emX}}

\usepackage{amsmath}
\usepackage{amssymb}
\usepackage{mathtools}
\usepackage{amsthm}

\usepackage[T1]{fontenc}    
\usepackage{hyperref}       
\usepackage{url}            
\usepackage{booktabs}       
\usepackage{nicefrac}       
\usepackage{microtype}      
\usepackage{amsmath}
\usepackage{amsfonts}
\usepackage{amsthm}
\usepackage{dsfont}
\usepackage{array}   
\usepackage{enumitem}
\usepackage{adjustbox}
\usepackage{booktabs}
\usepackage[dvipsnames,table]{xcolor}
\usepackage[dvipsnames,table]{xcolor}
\usepackage{multirow}
\usepackage{tabularx}
\usepackage{natbib}
\usepackage[most]{tcolorbox}
\usepackage{fancyvrb}
\usepackage{fvextra}
\setcitestyle{numbers,open={[},close={]}} 
\definecolor{blue1}{HTML}{D5E1EF}
\definecolor{blue2}{HTML}{9DC2D5}
\definecolor{blue3}{HTML}{5A94B9}
\definecolor{blue4}{HTML}{335372}
\definecolor{ysdarkpurple}{HTML}{4E2399}
\definecolor{ysshallowpurple}{HTML}{E6DBFF}
\definecolor{ysdarkred}{HTML}{8c2824}
\definecolor{ysshallowred}{HTML}{F8D7D7}
\definecolor{ysdarkblue}{HTML}{005E99}
\definecolor{ysshallowblue}{HTML}{CCEBFF}
\definecolor{ysdarkgrey}{HTML}{333333}
\definecolor{ysshallowgrey}{HTML}{E5E5E5}
\definecolor{boxblue1}{HTML}{9DC2D5}
\definecolor{boxblue2}{HTML}{D5E1EF}
\usepackage{etoc}
\usepackage{subcaption}
\definecolor{red4}{HTML}{E25659}
\definecolor{bestbg}{HTML}{CDD8E1}    
\definecolor{secondbg}{HTML}{DFE6EC}  
\definecolor{thirdbg}{HTML}{EEF2F5}   
\definecolor{oursbg}{HTML}{F2F5F8}    

\theoremstyle{plain}
\newtheorem{theorem}{Theorem}[section]

\theoremstyle{definition}
\newtheorem{definition}[theorem]{Definition}

\theoremstyle{remark}

\usepackage{pifont}

\usepackage[capitalize,noabbrev]{cleveref}
    
\begin{document}

\title{GraspLLM: Towards Zero-Shot Generalization on Text-Attributed Graphs with LLMs
}

\author{
\IEEEauthorblockN{
Hengyi Feng\textsuperscript{1,2*},
Zeang Sheng\textsuperscript{2*},
Meiyi Qiang\textsuperscript{2},
Yang Li\textsuperscript{3},
Wentao Zhang\textsuperscript{2,4$\dagger$}%
\thanks{* Equal contribution.}%
\thanks{$\dagger$ Corresponding author.}%
}

\vspace{0.3em}

\IEEEauthorblockA{
\textsuperscript{1}University of Electronic Science and Technology of China \\
\textsuperscript{2}Peking University \quad
\textsuperscript{3}Tencent Inc \quad
\textsuperscript{4}Zhongguancun Academy \\
hengyi.feng@std.uestc.edu.cn, \{shengzeang18, wentao.zhang\}@pku.edu.cn
}
}

\maketitle

\begin{abstract}
    Research on Text-Attributed Graphs (TAGs) has gained significant attention recently due to its broad applications across various real-world data scenarios, such as citation networks, e-commerce platforms, social media, and web pages.
    Inspired by the remarkable semantic understanding ability of Large Language Models (LLMs), there have been numerous attempts to integrate LLMs into TAGs.
    However, existing methods still struggle to generalize across diverse graphs and tasks, and their ability to capture transferable graph structural patterns remains limited.
    To address this, we introduce the GraspLLM, a framework that combines \underline{Gra}ph structural comprehension with \underline{s}emantic understanding \underline{p}rowess of \underline{LLM}s to enhance the cross-dataset and cross-task generalizability. 
    Specifically, we represent node texts from different graphs in a unified semantic space with a frozen general embedding model, on top of which we perform motif-aware contrastive learning across multiple motif-induced adjacency matrices to extract dataset-agnostic structural information.
    Then, with our proposed \textit{optimal contextual subgraph}, we extract the most contextually relevant subgraph for each target node and align these subgraphs to the token space of LLM via an alignment projector.
    Extensive experiments on TAG benchmark datasets spanning diverse domains reveal that GraspLLM consistently outperforms previous LLM-based methods for TAGs, especially in zero-shot scenarios, highlighting its strong generalizability across different datasets and tasks.
    Our code is available at \href{https://github.com/Heinz217/GraspLLM}{\url{https://github.com/Heinz217/GraspLLM}}.
\end{abstract}


\section{Introduction}

Graph data is ubiquitous in the real world, serving as a natural abstraction to represent complex relationships in various domains such as social networks~\cite{GraphAdapter, stock}, citation networks~\cite{OGB}, traffic networks~\cite{zhao2024stmgf}, and web page networks~\cite{Webkb}.
Beyond pure topology, many graphs are enriched with node-level textual information, forming Text-Attributed Graphs (TAGs)~\cite{textspace, survey2, survey3}.
For example, in e-commerce networks~\cite{CS-TAG}, nodes correspond to products, and their textual features are derived from user reviews. An edge between two products indicates that they are frequently co-purchased or co-viewed. 
Effectively analyzing such complex TAG data has therefore emerged as a pressing demand of graph data management pipelines.

In recent years, Large Language Models (LLMs) have achieved remarkable success across diverse domains and modalities~\cite{NLP1, NLP2, CV1, CV2, liang2025dataflow, bai2024survey},
demonstrating powerful generalization capabilities. 
This success has motivated researchers to explore the potential of adapting LLMs to graph machine learning on TAG data~\cite{GPT4Graph, Potential, noisy}. 
Current studies of LLMs on TAGs can be broadly categorized into two paradigms: \textbf{\textit{LLM as Enhancer}} and \textbf{\textit{LLM as Predictor}}~\cite{survey1}.

\textbf{\textit{1) LLM as Enhancer:}} This approach uses LLMs to augment textual information or refine feature representations for the input graphs before feeding them into Graph Neural Networks (GNNs)~\cite{opengraph, grenade, LLM-BP}, aiming to build more generalizable graph models~\cite{LLM-GNN, WalkLM, UniGLM, zhao2026dynamic}. 
However, the inherent limitations of GNNs persist: they remain constrained by the need for extensive task-specific training and show limited generalizability.
\textbf{\textit{2) LLM as Predictor:}} This approach converts graph data into LLM-friendly tokens or sequences~\cite{ Experts, TEA-GLM, GOFA} for LLMs to understand.
Yet, this approach also struggles to generalize across different graphs for its weak comprehension on graph structures.
Some methods~\cite{GraphAdapter, GraphPrompter} rely on single-node representations, which hinders LLMs from capturing structural knowledge. Other methods based on neighbor structure encoding~\cite{GraphGPT, LLaGA} focus on local structure, but lack a broader view of the graph. 

As summarized, the primary challenge of existing approaches is the \textbf{\textit{limited generalizability}}, which can be further decomposed into two key aspects:
\begin{itemize}
    \item \textbf{\textit{C1: Limited generalizability on features:}} 
    Significant distribution shifts in node features across different TAG datasets have been widely observed, which hinder both in-domain and cross-domain generalizabilities of models~\cite{Unigraph}.
    For example, although TEA-GLM~\cite{TEA-GLM} achieves competitive performance under in-domain zero-shot settings, 
    it suffers a relative performance drop of over 70\% (e.g., from 0.548 to 0.110) on several datasets under cross-domain settings.
    This highlights the necessity of developing approaches that can generalize well across diverse feature distributions of graph data.
    \item \textbf{\textit{C2: Limited generalizability on graph structures:}} 
    Real-world TAGs exhibit substantial structural diversity across datasets, 
    challenging models in transferring learned structural knowledge across datasets and tasks.
    A recent study~\cite{GOFA} highlights that current LLM-based approaches generally lack a deep comprehension of graph structures. 
    For instance, even for datasets from the same domain, their topological patterns such as clustering tendencies and connection densities vary significantly. Consequently, a model trained on one dataset frequently struggles to generalize even to other in-domain datasets.
\end{itemize}

To address these challenges, we propose \textbf{GraspLLM}, a novel framework that integrates \underline{Gra}ph structural comprehension with \underline{s}emantic understanding \underline{p}rowess of \underline{LLM}s to enhance the cross-dataset and cross-task generalizability.
The key idea lies in distilling node feature semantics and topological dependencies into compact subgraph representations, making node features and graph structures interpretable to LLMs.
Specifically, GraspLLM incorporates a carefully designed motif-aware self-supervised learning paradigm for GNN, built upon node features produced by a unified large embedding model that yields a domain-invariant feature space across datasets. 
The learned graph representations serve as structural priors, enabling GNNs to act as domain-agnostic structure extractors (\textbf{\textit{addressing C1}}).
Guided by GNN-derived structural insights, GraspLLM constructs \textbf{\textit{optimal contextual subgraphs}} via a greedy sampling algorithm, which prioritizes semantic and structural relevance, providing LLMs with richer and more informative subgraph representations. 
An alignment projector then transforms these subgraphs into the token space of LLMs (\textbf{\textit{addressing C2}}).
These designs enable LLMs to internalize and reason over graph structure, enhancing generalization. 
Extensive experiments demonstrate that GraspLLM excels in zero-shot scenarios, showcasing state-of-the-art cross-dataset and cross-task generalization capabilities.
Furthermore, in supervised settings, GraspLLM maintains competitive performance.

The main contributions and benefits of our work can be summarized as follows:
\begin{itemize}
    \item \textbf{\textit{Framework:}} 
    We propose GraspLLM, a novel framework that enables robust in-domain, cross-domain, and cross-task zero-shot generalization in TAGs with LLMs, effectively addressing key limitations in existing approaches.
    \item \textbf{\textit{Algorithms:}} 
    Leveraging the GNN-derived node representations, we design a subgraph sampling algorithm that provides LLMs with contextually relevant information, thereby enhancing their graph reasoning capabilities.
    \item \textbf{\textit{Experimental findings:}} 
    Evaluated on fourteen real-world TAG benchmark datasets across diverse domains (academic, e-commerce, social, etc.) and instantiated on diverse representative LLM backbones, GraspLLM consistently achieves state-of-the-art zero-shot performance, while maintaining competitive supervised performance. 
\end{itemize}

\section{Related Work}

\subsection{Self-Supervised Learning for Graph Neural Networks}
Graph Neural Networks (GNNs)~\cite{GCN, GAT, SAGE} have emerged as a powerful framework to efficiently process and analyze graph data. 
By harnessing the message-passing mechanism, GNNs effectively propagate and aggregate information across graph structures.
However, GNNs possess a poor task generalization capability~\cite{gnngeneral}.
To address this, self-supervised learning (SSL) is introduced to learn task-agnostic representations~\cite{GraphContrast, SimGRACE}.
Specifically, graph contrastive learning (GCL) has garnered particular attention
~\cite{gcc, GraphCL, zeng2025isgcl}
, focusing on learning representations by contrasting positive and negative samples. 
Notable methods in this domain are based on mutual information maximization~\cite{DGI, GRACE}, as well as those that employ whitening and decorrelation strategies~\cite{CCA-SSG}. 
However, their downstream performance typically relies on task-specific heads, limiting generalization across tasks and domains~\cite{TEA-GLM}.

\subsection{Large Language Models for Text-Attributed Graphs}

Adapting LLMs to graph machine learning has recently attracted growing interest, ranging from interactive reasoning over knowledge graphs~\cite{sun2024think, chen2024plan, wang2025agrag, afandi2026llm} to representation learning on TAGs. 
Existing efforts on TAGs can be broadly grouped into two main paradigms: \textbf{\textit{LLM as Enhancer}} and \textbf{\textit{LLM as Predictor}}.

\textbf{\textit{LLM as Enhancer}} approaches leverage LLMs to augment textual information or refine feature representations prior to GNN processing~\cite{WalkLM, simteg, ENGINE}.
TAPE~\cite{TAPE} utilizes the GPT model~\cite{GPT3} to generate predictions, which are then distilled into simplified insights to assist smaller models in decision making. 
OFA~\cite{OFA} and ZeroG~\cite{ZeroG} utilize LLMs to encode node text into the embedding space to enrich graph representations. 
GraphCLIP~\cite{GraphCLIP} generates natural language summaries of subgraphs to pretrain the model. 
LLM‑BP~\cite{LLM-BP} uses LLMs to produce task‑adaptive node embeddings (LLM2Vec~\cite{behnamghader2024llm2vec}) and estimate graph homophily for belief‑propagation‑based zero‑shot inference (GPT4o~\cite{hurst2024gpt4o}).
However, these approaches still inherit the limitations of GNNs.

\textbf{\textit{LLM as Predictor}} approaches convert graph data into LLM-interpretable tokens or sequences~\cite{graphprompt, GraphAdapter, GraphPrompter, graphagent}. 
For example, 
InstructGLM~\cite{InstructGLM} uses natural language to describe graph structures. 
GraphGPT~\cite{GraphGPT} and LLaGA~\cite{LLaGA} convert GNN-processed graphs into token sequences and map them to the LLM embedding space. 
TEA-GLM~\cite{TEA-GLM} and GOFA~\cite{GOFA} focus on node representations, aligning GNN output to the language space. However, these methods often struggle to capture structural dependencies, limiting the capacity of LLMs to reason over graph data~\cite{canlarge}. 
Moreover, their ability to generalize across tasks and domains remains insufficient.

\begin{figure*}[t]
    \centering
    \includegraphics[width=1.0\textwidth]{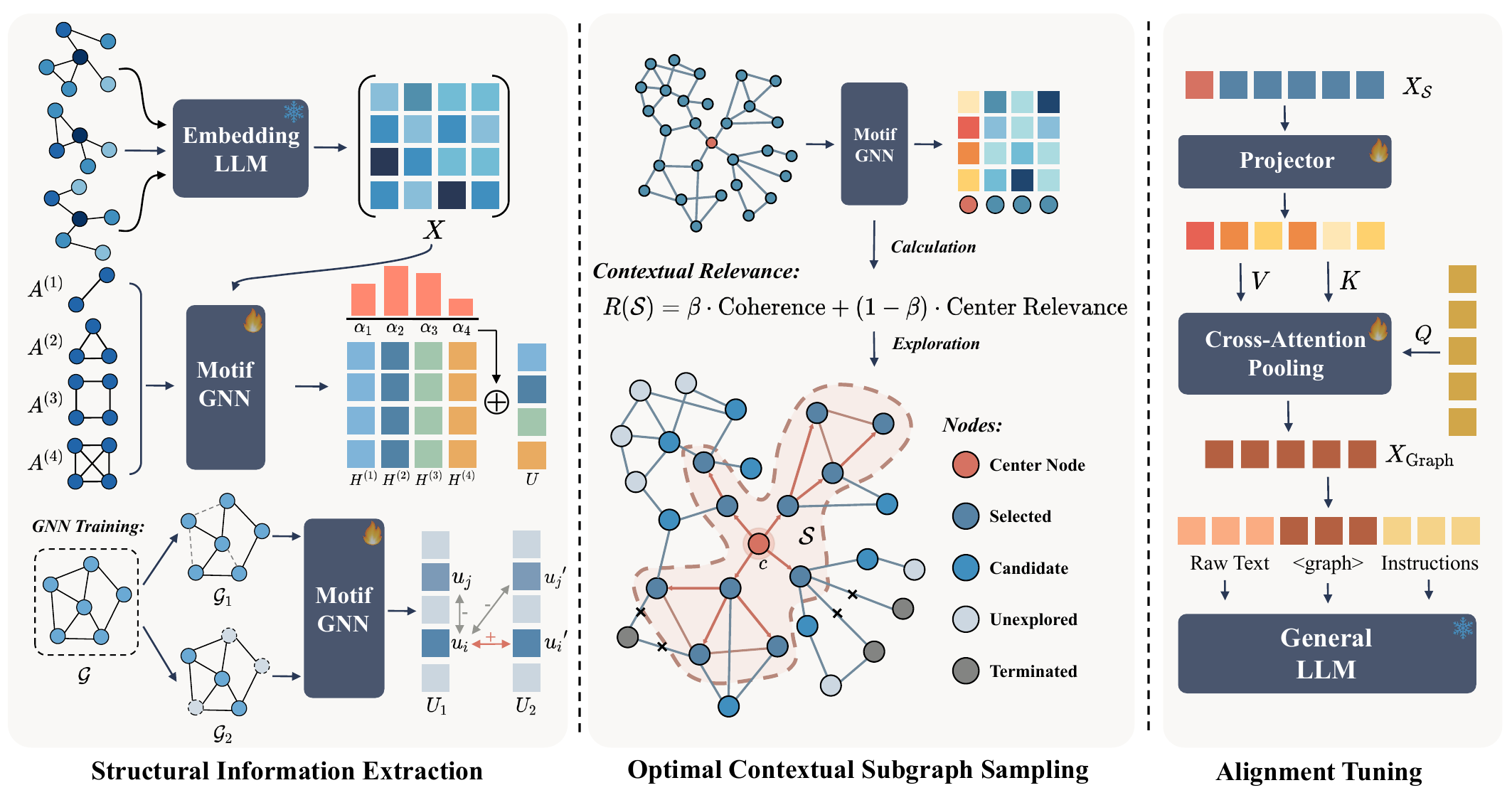}
    \caption{Overview of the GraspLLM framework. The full process is comprised of three stages: \textit{Structural Information Extraction} via a motif-aware GNN on top of a frozen embedding LLM, \textit{Optimal Contextual Subgraph Sampling} guided by contextual relevance, and \textit{Alignment Tuning} that projects the sampled subgraph into the token space of a frozen LLM.}
    \label{fig:overview}
\end{figure*}

\section{Methodology}
\label{sect:method}

\subsection{Overview}

In this section, we present the GraspLLM framework, which enables Large Language Models (LLMs) to effectively reason over graph-structured data with improved zero-shot generalization across datasets and tasks. 
As illustrated in Fig.~\ref{fig:overview}, our core idea is to distill topological dependencies and node-level semantics into compact subgraph representations. 
These representations are then projected into the token embedding space of LLMs. 

\subsection{Notations}
In this work, we focus on Text-Attributed Graphs (TAGs), where each node is associated with raw textual content, providing rich semantic information beyond structural connectivity.
Formally, a graph is defined as \(\mathcal{G}=(\mathcal{V}, \mathcal{E}, \mathbf{A}, \text{C})\), where \(\mathcal{V}=\{v_1,v_2,\dots,v_{|\mathcal{V}|}\}\) is the set of nodes, 
and \(\mathcal{E}=\{e_1,e_2,\dots,e_{|\mathcal{E}|}\}\) is the set of edges. The adjacency matrix is defined as \(\mathbf{A}\in \mathbb{R}^{N\times N}\), with \(\mathbf{A}_{ij}=1\) iff \((v_i,v_j)\in \mathcal{E}\). The textual content \(\text{C}={\{c_n\}}_{n=1}^{N}\) \((c_n\in \mathcal{T}^{{L}_n})\) represents the raw text for node \(n \in [1, 2, \ldots, N]\), where \(\mathcal{T}\) is the token dictionary, and \(L_n\) is the sequence length.

\subsection{Structural Information Extraction}
\label{sect:structural information}

GNNs provide a powerful framework for modeling graph structural dependencies. However, their reliance on domain-specific training limits their adaptability to real-world TAGs, which exhibit varying feature distributions. To address this, our aim is to develop a GNN that serves as a structural information encoder for LLMs, extracting graph insights while remaining adaptable across domains. 

\subsubsection{Unified Semantic Encoding}
\label{sect:unified-encoding}

A fundamental obstacle to cross-domain generalization is that node features from different graphs originate in distinct domains, leading to substantial distribution shifts that pure GNNs cannot reconcile. Here, we adopt Qwen3-Embedding-8B~\cite{zhang2025qwen3emb} as a single text encoder for datasets across all domains.

Formally, given a graph \(\mathcal{G}\) whose nodes are associated with raw text \(\mathrm{C}\), we obtain node-level feature representations as
\begin{equation}
\label{eq:qwen3-embed}
    \mathbf{X} = f_{\mathrm{Qwen3\text{-}Emb}}(\mathrm{C}) \in \mathbb{R}^{N\times d},
\end{equation}
where \(f_{\mathrm{Qwen3\text{-}Emb}}\) is kept frozen throughout the whole process. Since Qwen3-Embedding-8B is pretrained on a massive heterogeneous corpus covering web text, code, scientific literature and social content, the resulting embeddings already lie in a single semantic space that is empirically domain-invariant across academic, social, e-commerce and web-page graphs. 
This makes the following process genuinely dataset-agnostic.

\begin{figure}[h]
    \centering
    \resizebox{0.85\linewidth}{!}{
    \includegraphics{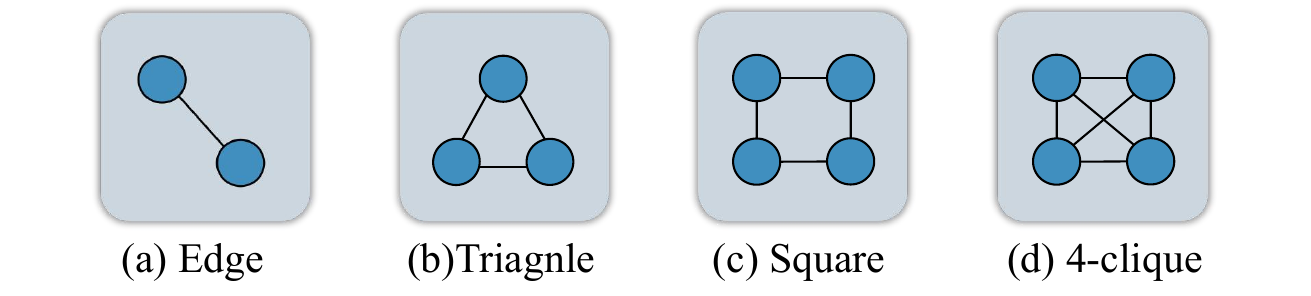}}
    \caption{Examples of the motifs.}
    \label{fig:motifs}
\end{figure}

\subsubsection{Motif-Aware Graph Self-Supervised Learning}
\label{sect:motif-aware}

Motif-defined neighborhoods, based on recurring subgraph patterns, capture essential structural information around nodes, offering a richer perspective than edge-defined neighborhoods. By focusing on stronger bonds, they help the models identify critical connections among nodes. Considering this property, we introduce a motif-aware graph self-supervised learning paradigm. 

We construct motif-based adjacency matrices \(\{\mathbf{A}^{(i)}\}_{i=1}^{M}\) for four graph motifs: edge, triangle, 4-cycle, and 4-clique (Fig.~\ref{fig:motifs}), reflecting structural relations from weak to strong. 
The GNN integrates these matrices through motif-specific message passing channels: In every message-passing step, for each motif channel \(i\), the compute process is as follows:
\begin{equation}
    \mathbf{H}_i^{(l)}=\sigma \left(\tilde{\mathbf{A}}^{(i)}\mathbf{H}^{(l-1)}\mathbf{W}_i^{(l)} \right),
\end{equation}
where \(\tilde{\mathbf{A}}^{(i)}=\mathbf{A}^{(i)}+\mathbf{I}\) is the self-loop augmented adjacency matrix, \(\mathbf{W}_{i}^{(l)}\) is the learnable transformation matrix, and \(\sigma(\cdot)\) denotes a nonlinear activation function.

The final node representation \(\mathbf{U}\) is computed as a weighted sum of the final-layer outputs \(\mathbf{H}_{i}^{(L)}\) from all \(M\) motif channels, with learnable weights \(\{\alpha_i\}_{i=1}^{M}\) controlling each motif's contribution: 
\begin{equation}
    \mathbf{U}=\sum_{i=1}^{M}\alpha_{i}\mathbf{H}_{i}^{(L)}.
\end{equation}

Our goal is to train a dataset-agnostic and task-agnostic GNN that captures generalizable structural patterns from graphs, rather than overfitting to a specific domain or task. 
To achieve this, we adopt an contrastive method, following~\cite{GRACE}, generating two augmented graph views \(\mathcal{G}_1\) and \(\mathcal{G}_2\). 
Given a node \(\boldsymbol{v_i}\), \(\boldsymbol{u_i} \in \mathcal{G}_1\) serves as the anchor, while \(\boldsymbol{u_i}^\prime \in \mathcal{G}_2\) serves as the positive sample. Embeddings of other nodes from either view are treated as intra-view and inter-view negatives. The contrastive objective is defined as:
\begin{equation}
\begin{aligned}
    \ell(\boldsymbol{u}_i, \boldsymbol{u}_i^\prime) = \log \frac{e^{\theta(\boldsymbol{u}_i, \boldsymbol{u}_i^\prime) / \tau}}{
    \begin{aligned}
        &e^{\theta(\boldsymbol{u}_i, \boldsymbol{u}_i^\prime) / \tau} + \sum_{j \neq i}^N e^{\theta(\boldsymbol{u}_i, \boldsymbol{u}_j) / \tau} \\
        &+ \sum_{j \neq i}^N e^{\theta(\boldsymbol{u}_i, \boldsymbol{u}_j^\prime) / \tau},
    \end{aligned}
    }
\end{aligned}
\end{equation}
where \(\mathds{1}_{[j \neq i]}\in \{ 0, 1\}\) is an indicator function that equals to 1 iff \(j \neq i\), \(\tau\) is a temperature parameter, and \(\theta(\cdot , \cdot)\) is the cosine similarity. The overall objective averages the loss in both directions:
\begin{equation}
    \mathcal{J} = \frac{1}{2N} \sum_{i=1}^{N} \left[ \ell\left(\boldsymbol{u_i}, \boldsymbol{u_i}^\prime\right) + \ell\left(\boldsymbol{u_i}^\prime, \boldsymbol{u_i}\right) \right].
\end{equation}
The GNN is trained end-to-end on top of the frozen Qwen3-Embedding-8B features within this contrastive framework.

\subsection{Optimal Contextual Subgraph Sampling}
\label{sect:subgraph sampling}

In TAGs, the properties of a node can have strong connections with its local and global structure. It is crucial to provide more comprehensive and effective subgraph information for LLMs to enable better understanding of the graph data. Therefore, a key challenge arises:
\textbf{\textit{How to perform subgraph sampling to effectively preserve essential contextual and structural information for LLMs?}} 
To address this, we propose the concept of \textit{Contextual Relevance}.

\begin{definition}[\textbf{Contextual Relevance}]
\label{definition:contextual relevance}
\textit{Let \(\boldsymbol{u_n}\in \mathbb{R}^{d}\) denote the embedding vector of node \(n\). The contextual relevance \(R(\mathcal{S})\) for the center node \(c\) and a set of nodes \(\mathcal{S}\in \mathcal{V}\) is defined as:}
\begin{multline}
    R(\mathcal{S}) = \beta \cdot \underbrace{\sum_{n \in \mathcal{S}} \sum_{m \in \mathcal{N}(n)} \max\left(0, \cos(\boldsymbol{u_n}, \boldsymbol{u_m}) \right)}_{\text{Structural Coherence}} + \\
    (1-\beta) \cdot \underbrace{\sum_{n \in \mathcal{S}} \max\left(0, \cos(\boldsymbol{u_n}, \boldsymbol{u_c}) \right)}_{\text{Center Relevance}},
\end{multline}
\textit{where \(\beta \in [ 0,\ 1]\), controlling the trade-off between coherence and relevance, \(\cos(\cdot, \cdot)\) is the cosine similarity, and \(\mathcal{N}(n)\) is the neighborhood of \(n\) in \(\mathcal{G}\).}
\end{definition}

The function \(R(\mathcal{S})\) captures both structural and semantic relevance, balancing local coherence and proximity to the central node. Based on this, we define the \textbf{\textit{Optimal Contextual Subgraph}} \(\mathcal{G'}\subseteq \mathcal{G}\) for a given center node \(c\).

\begin{definition}[\textbf{Optimal Contextual Subgraph}]
\label{definition:subgraph}
\textit{Optimal Contextual Subgraph \(\mathcal{G'}\) is defined as the set of nodes that maximizes contextual relevance \(R(\mathcal{S})\) for a center node \(c\):}
\begin{equation}
   \mathcal{G'}=\arg\max_{\mathcal{S}\in \mathcal{V}}R(\mathcal{S}) \ \ \ s.t. |\mathcal{S}| \leq L, \ d(v, \ c)\leq T, \ \forall \ v\in \mathcal{S}, 
\end{equation}
\textit{where \(d(v, c)\) represents the shortest path distance constraint, \(T\) is the maximum depth of exploration, and \(L\) is the upper limit on the number of selected nodes.}
\end{definition}

To effectively construct such a subgraph \(\mathcal{G'}\), we propose a greedy node selection algorithm, guided by GNN-derived node embeddings. 
The core idea of the algorithm is to perform relevance-aware sampling in a depth-first manner, treating each neighbor of the center node as the starting node. 
At each step, the node \(v\) yielding the highest marginal gain \(\Delta R(v \mid \mathcal{S})\) is selected for inclusion in \(\mathcal{S}\). 
A threshold \(\tau\) dynamically halts the action to prevent redundant expansion:
\begin{equation}
\label{eq:eta}
    \eta(v \mid \mathcal{S}) = \frac{\Delta R(v \mid \mathcal{S})}{\sum_{u \in \mathcal{N}(\mathcal{S}) \setminus \mathcal{S}} \Delta R(u \mid \mathcal{S})}.
\end{equation}
If \(\eta(v\mid \mathcal{S})<\tau\) for all candidate nodes \(v\), the search terminates early. The full procedure is summarized in Algorithm~\ref{algorithm:graph}. To further analyze the effectiveness of the proposed algorithm, we first consider an ideal greedy selection over the feasible candidate set.

\begin{algorithm}[t]
\caption{Greedy Node Selection}
\label{algorithm:graph}
\begin{algorithmic}[1]
\REQUIRE node embeddings $\{\boldsymbol{u}_n\}_{n\in\mathcal{V}}$, center node $c$, threshold $\tau$, max depth $T$, capacity $L$
\ENSURE contextual subgraph node set $\mathcal{S}$
\STATE $\mathcal{S}\gets\{c\}\cup\mathcal{N}(c)$;\ \ $\mathcal{V}_{\text{vis}}\gets\mathcal{S}$
\FOR{each $n\in\mathcal{N}(c)$}
    \STATE $v \gets n$
    \FOR{$t=1,\dots,T$}
        \STATE $\mathcal{N}'\gets\mathcal{N}(v)\setminus\mathcal{V}_{\text{vis}}$
        \IF{$\mathcal{N}'=\varnothing$ \textbf{ or } $|\mathcal{S}|\geq L$}
            \STATE \textbf{break}
        \ENDIF
        \STATE compute $\eta(u\mid\mathcal{S})$ for each $u\in\mathcal{N}'$ via \eqref{eq:eta}
        \STATE $\mathcal{N}^{\star}\gets\{u\in\mathcal{N}':\eta(u\mid\mathcal{S})>\tau\}$
        \IF{$\mathcal{N}^{\star}=\varnothing$}
            \STATE \textbf{break}
        \ENDIF
        \STATE $v\gets\arg\max_{u\in\mathcal{N}^{\star}}\eta(u\mid\mathcal{S})$
        \STATE $\mathcal{S}\gets\mathcal{S}\cup\{v\}$;\ \ $\mathcal{V}_{\text{vis}}\gets\mathcal{V}_{\text{vis}}\cup\{v\}$
    \ENDFOR
\ENDFOR
\RETURN $\mathcal{S}$
\end{algorithmic}
\end{algorithm}

\begin{theorem}[\textbf{Greedy Principle}]
\label{theorem:approximation}
\textit{The greedily selected subgraph \(\mathcal{S}\) achieves at least a \(1-\frac{1}{e}\) approximation to the optimal contextual subgraph \(\mathcal{G'}\) with respect to the objective of maximizing contextual relevance.} 
\end{theorem}

\begin{proof}[Proof sketch]
It suffices to show that $R$ is monotone and submodular; the $(1-1/e)$ bound then follows from~\cite{Nemhauser}. For any $\mathcal{S}\subseteq\mathcal{V}$ and $v\notin\mathcal{S}$, the marginal gain admits the closed form:
\begin{multline}
\label{eq:marginal-gain}
\Delta R(v\mid\mathcal{S}) = \beta\!\sum_{m\in\mathcal{N}(v)}\!\max\!\big(0,\cos(\boldsymbol{u_v},\boldsymbol{u_m})\big) + \\
(1\!-\!\beta)\max\!\big(0,\cos(\boldsymbol{u_v},\boldsymbol{u_c})\big),
\end{multline}
which is independent of $\mathcal{S}$ and nonnegative. Independence yields modularity, so $\Delta R(v\mid\mathcal{S})=\Delta R(v\mid\mathcal{S}')$ for all $\mathcal{S}\subseteq\mathcal{S}'$, which trivially satisfies submodularity. Nonnegativity yields monotonicity, $R(\mathcal{S}\cup\{v\})\geq R(\mathcal{S})$. Applying~\cite{Nemhauser} under the cardinality constraint $|\mathcal{S}|\leq L$ then gives the desired bound
\begin{equation}
\label{eq:greedy-bound}
    R(\mathcal{S}) \geq \left(1-\tfrac{1}{e}\right) R(\mathcal{G}').
\end{equation}
\end{proof}

Algorithm~\ref{algorithm:graph} follows this greedy principle with a depth-first restricted search to reduce per-step enumeration cost.

Beyond approximation quality, we further justify that feeding the sampled subgraph to the LLM is information-theoretically beneficial, provided that $\mathcal{S}$ faithfully approximates $\mathcal{G}'$ (\textit{Fidelity}) and $\mathcal{G}'$ carries information complementary to the raw text $\mathrm{X}$ (\textit{Non-redundancy}).

\begin{theorem}[\textbf{Informativeness}]
\label{theorem:informativeness}
\textit{Let $y$ be the prediction target, $\mathrm{X}$ the raw node text, and $\mathcal{G}'$ the optimal contextual subgraph. If $H(\mathcal{G}'\mid\mathcal{S})\leq\epsilon$ (\textit{Fidelity}) and $H(y\mid\mathrm{X},\mathcal{G}')=H(y\mid\mathrm{X})-\epsilon'$ with $\epsilon'>\epsilon$ (\textit{Non-redundancy}), then incorporating $\mathcal{S}$ strictly reduces predictive uncertainty:}
\begin{equation}
    H(y\mid\mathrm{X},\mathcal{S}) < H(y\mid\mathrm{X}).
\end{equation}
\end{theorem}

\begin{proof}[Proof sketch]
Following the entropy-based analysis of~\cite{TAPE}, we introduce $\mathcal{G}'$ as an auxiliary variable and decompose $H(y\mid\mathrm{X},\mathcal{S})$ via the chain rule of entropy:
\begin{equation}
\label{eq:chain}
H(y \mid \mathrm{X},\mathcal{S}) = H(y \mid \mathrm{X},\mathcal{G}',\mathcal{S}) + I(y;\mathcal{G}' \mid \mathrm{X},\mathcal{S}).
\end{equation}
For the mutual-information term, by the non-negativity of conditional entropy,
\begin{equation}
\label{eq:mi-bound}
I(y;\mathcal{G}'\mid\mathrm{X},\mathcal{S}) = H(\mathcal{G}'\mid\mathrm{X},\mathcal{S}) - H(\mathcal{G}'\mid y,\mathrm{X},\mathcal{S}) \leq H(\mathcal{G}'\mid\mathrm{X},\mathcal{S}).
\end{equation}
Applying ``conditioning reduces entropy'' to the two terms on the right-hand side of (\ref{eq:chain})--(\ref{eq:mi-bound}) yields
\begin{equation}
\label{eq:cond-reduce}
H(y\mid\mathrm{X},\mathcal{G}',\mathcal{S}) \leq H(y\mid\mathrm{X},\mathcal{G}'),\quad
H(\mathcal{G}'\mid\mathrm{X},\mathcal{S}) \leq H(\mathcal{G}'\mid\mathcal{S}),
\end{equation}
so that combining (\ref{eq:chain})--(\ref{eq:cond-reduce}) gives the key bound
\begin{equation}
\label{eq:key-bound}
H(y\mid\mathrm{X},\mathcal{S}) \leq H(y\mid\mathrm{X},\mathcal{G}') + H(\mathcal{G}'\mid\mathcal{S}).
\end{equation}
Finally, invoking \textit{Non-redundancy} on the first term of (\ref{eq:key-bound}) and \textit{Fidelity} on the second,
\begin{equation}
H(y\mid\mathrm{X},\mathcal{S}) \leq \big[H(y\mid\mathrm{X}) - \epsilon'\big] + \epsilon = H(y\mid\mathrm{X}) - (\epsilon' - \epsilon),
\end{equation}
which is strictly less than $H(y\mid\mathrm{X})$ by the assumption $\epsilon'>\epsilon$, completing the proof.
\end{proof}

Taken together, the two theorems offer end-to-end justification for our subgraph sampler: Theorem~\ref{theorem:approximation} guarantees that the greedy procedure produces a subgraph that is provably close to the optimal contextual subgraph $\mathcal{G}'$ in terms of contextual relevance, while Theorem~\ref{theorem:informativeness} ensures that any subgraph sufficiently close to $\mathcal{G}'$ injects task-relevant information not contained in the raw node text, thereby strictly reducing the LLM's predictive uncertainty. In other words, ``approximation quality'' on the sampling objective and ``informativeness'' on the downstream task are bridged by a single greedy procedure, making the use of $\mathcal{S}$ as the LLM context theoretically well-founded rather than purely heuristic.

This algorithm facilitates the extraction of deeper graph insights. The final node sequence can be computed as follows:
\begin{equation}
\begin{split}
    \mathcal{S} = \{c\} \cup \mathcal{N}(c) \cup \bigcup_{v \in \mathcal{N}(c)} \big\{ &v_1, v_2, \dots, v_T \mid \\
    &v_t = \arg \max_{v \in \mathcal{N}(v_{t-1})} \eta(v \mid \mathcal{S}_{t-1}) \big\}.
\end{split}
\end{equation}

\subsection{Alignment Tuning}
\label{sect:alignment tuning}

To improve the graph understanding ability of LLMs, we align the node embedding space with the token space via an alignment projector while addressing distribution shifts across different domains through a tailored normalization strategy. 
Notably, we omit variance scaling to preserve embedding magnitude, which is vital for feature representation.
Formally, for a subgraph \(\mathcal{S}\) with node features \(\mathbf{X}_{\mathcal{S}}\in \mathbb{R}^{N\times d}\):
\begin{equation}
    \mathbf{Z} = f_{\theta}\left(\mathbf{X}_{\mathcal{S}} - \boldsymbol{\mu} \right),
\end{equation}
where \(f_{\theta}(\cdot)\) denotes an MLP and \(\boldsymbol{\mu}=\frac{1}{N}\sum_{i=1}^{N}\mathbf{X}_{\mathcal{S}}^{(i)}\) represents feature-wise mean.

Directly feeding projected graph tokens into LLMs can result in excessively long input sequences, increasing the risk of overfitting to specific data patterns
rather than learning generalizable representations. To address this, we apply an attention-based pooling mechanism after projection to compress the sequence while preserving critical information. Specifically, scaled dot-product cross-attention is employed between learnable queries \(\mathbf{Q} \in \mathbb{R}^{T \times d}\) and input tokens \(\mathbf{Z}\):
\begin{align}
     \mathbf{X}_{\text{Graph}} = \operatorname{Softmax}\left( \frac{ \mathbf{Q} \mathbf{Z}^{\top} }{ \tau \sqrt{d} } \right) \mathbf{Z}.
\end{align}
Here, learnable temperature \(\tau\) and scaling factor \(\sqrt{d}\) stabilize attention, enabling the model to dynamically attend to the most informative elements while producing a fixed-length representation, regardless of the input graph size.

We employ an instruction tuning paradigm, organizing questions and answers in a chat format, with subgraph representations embedded via a \texttt{<graph>} token. The prompt templates for zero-shot node classification and link prediction are shown in Fig.~\ref{prompt:node_classification} and Fig.~\ref{prompt:link_prediction}, respectively.

\definecolor{promptbg}{HTML}{DCE3EA}
\definecolor{promptaccent}{HTML}{6E8499}
\definecolor{promptred}{HTML}{A0524A}
\begin{figure}[t]
\begin{tcolorbox}[
    enhanced,
    colback=promptbg, colframe=promptaccent,
    arc=2pt, boxrule=0.4pt,
    fontupper=\footnotesize,
    title={\textit{Node Classification}},
    coltitle=white, colbacktitle=promptaccent,
    fonttitle=\footnotesize\bfseries,
    left=5pt, right=5pt, top=3pt, bottom=3pt,
    boxsep=2pt
]
Given a node-centered graph: \textcolor{promptred}{<graph>}, each node represents \textcolor{promptred}{<node\_description>}. The 0th node is the central node, with the following information: \textcolor{promptred}{<text>}. The graph consists of the center node, the neighbors of the center node, and the context-relevant nodes. We need to classify the center node into \textcolor{promptred}{<num\_classes>} classes: \textcolor{promptred}{<class\_name>}, please tell me which class the center node belongs to?
\end{tcolorbox}
\caption{Prompt template for zero-shot node classification.}
\label{prompt:node_classification}
\end{figure}

\begin{figure}[t]
\begin{tcolorbox}[
    enhanced,
    colback=promptbg, colframe=promptaccent,
    arc=2pt, boxrule=0.4pt,
    fontupper=\footnotesize,
    title={\textit{Link Prediction}},
    coltitle=white, colbacktitle=promptaccent,
    fonttitle=\footnotesize\bfseries,
    left=5pt, right=5pt, top=3pt, bottom=3pt,
    boxsep=2pt
]
Given two node-centered graphs: For the first graph \textcolor{promptred}{<graph1>}, the central node (node1) has: \textcolor{promptred}{<text1>}. For the second graph \textcolor{promptred}{<graph2>}, the central node (node2) has: \textcolor{promptred}{<text2>}. The graph consists of the center node, the neighbors of the center node, and the context-relevant nodes. We need to classify whether these two nodes are connected in the original graph, please answer Yes or No.
\end{tcolorbox}
\caption{Prompt template for zero-shot link prediction.}
\label{prompt:link_prediction}
\end{figure}

Using Vicuna-7B-v1.5~\cite{vicuna} as the foundation model, we freeze the LLM and optimize only the projector. 
The training objective is to maximize the likelihood of correct answer generation:
\begin{equation}
    \max_{\theta}\prod_{i=1}^{N}p(\mathbf{X}_{\mathrm{answer}}^{(i)} \mid \mathbf{X}_{\mathrm{question}}^{(i)}, \mathbf{X}_{\mathrm{graph}}^{(i)}\ ;\theta ).
\end{equation}

\subsection{Engineering Optimizations for Large Graphs}
\label{sect:engineering}

To be applicable to real-world graph data systems, GraspLLM must remain practical on large TAGs with up to millions of nodes.
The \textit{\textbf{Optimal Contextual Subgraph}} sampling stage is by far the dominant bottleneck at this scale, so we devote a set of dedicated engineering optimizations to it.
The implementation of the sampler in Section~\ref{sect:subgraph sampling} is convenient for the moderate-scale benchmarks studied so far, but a direct port to large TAGs is impractical. 
To make GraspLLM operate at scale without altering its algorithmic semantics, we introduce three layers of system-level optimization. 
Crucially, each step still uses the same processing as Algorithm~\ref{algorithm:graph}; only the data layout and dispatch are different.

\paragraph{\textbf{Algorithmic restructuring}}
We first re-express the per-center sampling loop in a fully vectorized form. The walks anchored at a center $c$ share the same center embedding and only differ in their visited sets. 
We therefore broadcast them along a leading batch dimension $B{=}9$ and replace the candidate set $\mathcal{S}$ by two persistent boolean buffers $\mathtt{visited},\mathtt{cand}\in\{0,1\}^{B\times N}$. Updating $\mathcal{S}$ then reduces to a scatter, taking the open neighborhood reduces to a bitwise AND, and the per-step argmax reduces to a single \texttt{top-k} call. To avoid materializing $[B,N]$ priority tensors on million-node graphs, we further cap candidate ranking at the top $c_{\max}$ candidates per walk (default $c_{\max}{=}256$). On all paper benchmarks this cap is never reached, and the resulting subgraphs are distributionally indistinguishable from the reference implementation.

\paragraph{\textbf{GPU-native data layout}}
We materialize the graph adjacency directly on the GPU as a Compressed Sparse Row (CSR) triple $(\mathtt{nb\_flat},\,\mathtt{ptr},\,\mathtt{deg})$, built once at startup via a single \texttt{argsort}/\texttt{bincount}/\texttt{cumsum} pipeline and reused across all centers. Neighbor lookup then collapses into a contiguous slice $\mathtt{nb\_flat}[\mathtt{ptr}[v]\!:\!\mathtt{ptr}[v{+}1]]$, removing both Python-side overhead and per-step memory allocation. The persistent boolean buffers above are similarly allocated once per worker and reset in place via \texttt{zero\_()}. 
We keep the frozen embeddings in fp16, which roughly halves the memory and bandwidth footprint at no measurable cost in scoring fidelity.

\paragraph{\textbf{Multi-GPU parallelization}}
Because the sampler is invoked independently per center node, scaling out across $P$ GPUs is realized as an embarrassingly parallel partition of the center set, with one worker process per GPU and per-worker GPU-CSR replicas. 
To keep per-GPU memory bounded as $N$ grows, we additionally support a CUDA-IPC--based topology in which a single owner GPU loads the embedding once and exposes it to peer workers via a shared CUDA storage handle, removing the $20$\,GB fp16 embedding replica on non-owner cards. To further reduce kernel-launch overhead when the per-step GPU work is small, we batch $K$ centers per worker forward pass (default $K{=}4$), amortizing launch cost across $9K$ walks per step.

\section{Experiments}

In this section, we conduct comprehensive experiments to validate the effectiveness of GraspLLM across various settings, addressing the following research questions:
\begin{itemize}
    \item \textbf{RQ1:} How well does GraspLLM generalize under in-domain and cross-domain zero-shot settings?
    \item \textbf{RQ2:} How well does GraspLLM perform when faced with an unseen task?
    \item \textbf{RQ3:} How does the \textit{Optimal Contextual Subgraph Sampling} contribute to the zero-shot inference?
    \item \textbf{RQ4:} How does each component of GraspLLM contribute to the final performance?
\end{itemize}

\subsection{Experimental Settings}
\label{sect:settings}

\paragraph{Datasets.} We evaluate the effectiveness of GraspLLM in fourteen widely used real-world text-attributed graphs spanning five diverse domains: 
1) \textbf{Citation network}: Cora~\cite{Cora}, Citeseer~\cite{Citeseer}, Pubmed~\cite{Pubmed} and Arxiv~\cite{OGB};
2) \textbf{E-commerce}: Books-History (History), Ele-Computer (Computer) and Ele-Photo (Photo)~\cite{CS-TAG};
3) \textbf{Wikipedia pages}: WikiCS~\cite{Wikics};
4) \textbf{Social network}: Reddit~\cite{GraphAdapter} and Instagram~\cite{GraphAdapter};
5) \textbf{Web pages}: Cornell, Texas, Washington and Wisconsin~\cite{Webkb}.
Key statistics of the fourteen datasets are summarized in Table~\ref{tab:dataset-stats}, covering scales from $187$ to $169{,}343$ nodes, average degrees from $2.65$ to $36.94$, and global clustering coefficients spanning two orders of magnitude ($0.01$--$0.26$), which collectively reflect the structural heterogeneity that GraspLLM is designed to handle.

\begin{table}[h]
    \centering
    \caption{Statistics of the fourteen TAG benchmarks used in our experiments. \textbf{Avg.\,Deg.}, \textbf{Tri./Node}, and \textbf{Global Clust.} denote the average node degree, the average number of triangles per node, and the global clustering coefficient, respectively. \textbf{\#Cls.} is the number of node-classification classes.}
    \label{tab:dataset-stats}
    \scalebox{0.80}{
    \begin{tabular}{l|rrrrrc}
    \toprule
    \textbf{Dataset} & \textbf{\#Nodes} & \textbf{\#Edges} & \textbf{Avg.\,Deg.} & \textbf{Tri./Node} & \textbf{Global Clust.} & \textbf{\#Cls.} \\
    \midrule
    \multicolumn{7}{l}{\cellcolor{oursbg}\textit{\textbf{Citation}}} \\
    \midrule
    Cora       & 2{,}708   & 5{,}278     & 3.89  & 1.80   & 0.0935 & 7  \\
    Citeseer   & 3{,}186   & 4{,}225     & 2.65  & 0.99   & 0.1291 & 6  \\
    Pubmed     & 19{,}717  & 43{,}244    & 4.49  & 1.90   & 0.0537 & 3  \\
    Arxiv      & 169{,}343 & 1{,}157{,}799 & 13.67 & 39.56  & 0.0162 & 40 \\
    \midrule
    \multicolumn{7}{l}{\cellcolor{oursbg}\textit{\textbf{E-commerce}}} \\
    \midrule
    Computer   & 87{,}229  & 628{,}274   & 14.40 & 60.64  & 0.1025 & 10 \\
    Photo      & 48{,}362  & 436{,}902   & 18.06 & 99.30  & 0.1197 & 12 \\
    History    & 41{,}551  & 251{,}590   & 12.10 & 94.28  & 0.1709 & 12 \\
    \midrule
    \multicolumn{7}{l}{\cellcolor{oursbg}\textit{\textbf{Social}}} \\
    \midrule
    Instagram  & 11{,}339  & 83{,}344    & 14.70 & 36.92  & 0.1638 & 2  \\
    Reddit     & 33{,}434  & 167{,}670   & 10.02 & 5.35   & 0.0103 & 2  \\
    \midrule
    \multicolumn{7}{l}{\cellcolor{oursbg}\textit{\textbf{Web link}}} \\
    \midrule
    WikiCS     & 11{,}701  & 216{,}123   & 36.94 & 826.69 & 0.2623 & 10 \\
    \midrule
    \multicolumn{7}{l}{\cellcolor{oursbg}\textit{\textbf{Web page (School)}}} \\
    \midrule
    Cornell    & 191 & 274 & 2.86 & 0.84 & 0.0323 & 5 \\
    Texas      & 187 & 298 & 3.18 & 1.07 & 0.0327 & 5 \\
    Wisconsin  & 265 & 459 & 3.46 & 1.35 & 0.0397 & 5 \\
    Washington & 229 & 416 & 3.63 & 1.29 & 0.0356 & 5 \\
    \bottomrule
    \end{tabular}
    }
    \vspace{-1mm}
\end{table}

\begin{table*}[h!]
    \centering
    \caption{In-domain zero-shot node classification accuracy on ten unseen TAG datasets, organized by domain (citation, e-commerce, weblink, social, and webpage). The rightmost \textbf{AVG} column summarizes each method's overall in-domain generalization by averaging accuracies over the ten datasets. Per-column ranking is indicated by cell background: \colorbox{bestbg}{best}, \colorbox{secondbg}{second-best}, \colorbox{thirdbg}{third-best}.}
    \label{tab:nc_citation}
    \scalebox{1}{
    \begin{tabular}{c|ccc|cc|c|c|ccc|c}
    \toprule
    \multirow{2}{*}{\textbf{Model}} & \multicolumn{3}{c|}{\textbf{Citation}} & \multicolumn{2}{c|}{\textbf{E-commerce}} & \textbf{Weblink} & \textbf{Social} & \multicolumn{3}{c|}{\textbf{Webpage}} & \multirow{2}{*}{\textbf{AVG}} \\
    \cmidrule{2-11}
    & \textbf{Cora} & \textbf{Citeseer} & \textbf{Pubmed} & \textbf{History} & \textbf{Photo} & \textbf{WikiCS} & \textbf{Instagram} & \textbf{Texas} & \textbf{Washington} & \textbf{Cornell} & \\
    \midrule
    MLP & 0.173 & 0.158 & 0.493 & 0.153 & 0.201 & 0.177 & 0.236 & 0.253 & 0.340 & 0.256 & 0.244 \\
    \midrule
    GCN & 0.180 & 0.151 & 0.498 & 0.152 & 0.253 & 0.142 & 0.242 & 0.389 & 0.474 & 0.340 & 0.282 \\
    GAT & 0.090 & 0.140 & 0.481 & 0.245 & 0.412 & 0.170 & 0.368 & 0.526 & 0.340 & 0.359 & 0.313 \\
    GraphSAGE & 0.122 & 0.098 & 0.383 & 0.157 & 0.149 & 0.213 & 0.332 & 0.552 & 0.277 & 0.359 & 0.264 \\
    \midrule
    NodeFormer & 0.053 & 0.154 & 0.317 & 0.179 & 0.152 & 0.383 & 0.311 & 0.211 & 0.213 & 0.359 & 0.233 \\
    DIFFormer & 0.124 & 0.189 & 0.352 & 0.253 & 0.291 & 0.312 & 0.367 & 0.315 & 0.305 & 0.296 & 0.280 \\
    \midrule
    BERT & 0.193 & 0.254 & 0.361 & 0.205 & 0.273 & 0.305 & 0.531 & 0.471 & 0.439 & 0.440 & 0.347 \\
    RoBERTa & 0.253 & 0.291 & 0.259 & 0.293 & 0.306 & 0.253 & 0.465 & 0.423 & 0.415 & 0.426 & 0.338 \\
    E5 & 0.424 & 0.446 & 0.423 & 0.367 & 0.324 & 0.325 & 0.496 & 0.501 & 0.518 & 0.465 & 0.429 \\
    Sentence-Bert & 0.318 & 0.489 & 0.428 & 0.312 & 0.336 & 0.343 & 0.484 & 0.432 & 0.547 & 0.628 & 0.432 \\
    \midrule
    Qwen2-7B & 0.566 & 0.541 & 0.651 & 0.434 & 0.457 & 0.318 & 0.243 & 0.453 & 0.587 & 0.526 & 0.478 \\
    LLaMA-2-7B & 0.495 & 0.348 & 0.640 & 0.354 & 0.397 & 0.312 & 0.359 & 0.463 & 0.483 & 0.445 & 0.430 \\
    LLaMA-3.1-8B & 0.532 & 0.655 & 0.689 & 0.376 & 0.429 & 0.356 & 0.417 & 0.442 & 0.438 & 0.562 & 0.490 \\
    Vicuna-7B-v1.5 & 0.490 & 0.381 & 0.632 & 0.324 & 0.369 & 0.349 & 0.320 & 0.497 & 0.463 & 0.560 & 0.439 \\
    Vicuna-7B-IT & 0.513 & 0.410 & 0.663 & 0.343 & 0.381 & 0.374 & 0.356 & 0.513 & 0.489 & 0.620 & 0.466 \\
    \midrule
    OFA & 0.171 & 0.417 & 0.379 & 0.052 & 0.340 & 0.359 & 0.406 & 0.135 & 0.087 & 0.203 & 0.255 \\
    ZeroG & 0.603 & 0.534 & 0.732 & 0.327 & 0.463 & 0.627 & 0.507 & 0.236 & 0.105 & 0.109 & 0.424 \\
    GraphCLIP & 0.602 & 0.627 & 0.388 & 0.491 & 0.368 & 0.614 & \cellcolor{secondbg}0.571 & 0.185 & 0.279 & 0.234 & 0.436 \\
    UniGLM & 0.456 & 0.523 & 0.432 & 0.442 & 0.376 & 0.551 & 0.394 & 0.384 & 0.310 & 0.230 & 0.410 \\
    LLM-BP & \cellcolor{bestbg}\textbf{0.714} & \cellcolor{bestbg}\textbf{0.701} & 0.768 & 0.595 & 0.519 & 0.680 & 0.483 & 0.794 & 0.700 & \cellcolor{secondbg}0.849 & 0.680 \\
    \midrule
    GraphGPT & 0.381 & 0.373 & 0.701 & 0.153 & \cellcolor{bestbg}\textbf{0.762} & 0.536 & 0.439 & 0.585 & 0.676 & 0.602 & 0.521 \\
    LLaGA & 0.348 & 0.347 & 0.793 & 0.363 & 0.392 & 0.710 & 0.479 & 0.616 & 0.625 & 0.714 & 0.539 \\
    TEA-GLM & 0.379 & 0.402 & 0.848 & 0.528 & 0.497 & 0.691 & 0.508 & 0.791 & 0.762 & 0.805 & 0.621 \\
    GOFA & \cellcolor{secondbg}0.711 & 0.657 & 0.748 & 0.563 & 0.487 & 0.563 & 0.418 & 0.384 & 0.310 & 0.395 & 0.536 \\
    \midrule
    \multicolumn{12}{l}{\cellcolor{oursbg}\textit{\textbf{GraspLLM (Ours)}}} \\
    \midrule
    \textbf{Vicuna-7B-v1.5} & 0.614 & \cellcolor{thirdbg}0.681 & \cellcolor{secondbg}0.869 & \cellcolor{secondbg}0.658 & 0.554 & \cellcolor{thirdbg}0.855 & \cellcolor{secondbg}0.571 & \cellcolor{secondbg}0.868 & \cellcolor{secondbg}0.872 & \cellcolor{bestbg}\textbf{0.897} & \cellcolor{thirdbg}0.740 \\
    \textbf{Mistral-7B-v0.3} & 0.613 & 0.618 & 0.819 & 0.527 & 0.627 & 0.837 & \cellcolor{bestbg}\textbf{0.636} & \cellcolor{thirdbg}0.842 & \cellcolor{bestbg}\textbf{0.894} & 0.821 & 0.723 \\
    \textbf{LLaMA-3.1-8B} & \cellcolor{thirdbg}0.635 & 0.641 & \cellcolor{thirdbg}0.850 & \cellcolor{thirdbg}0.640 & \cellcolor{thirdbg}0.633 & \cellcolor{secondbg}0.866 & \cellcolor{thirdbg}0.543 & \cellcolor{bestbg}\textbf{0.895} & \cellcolor{thirdbg}0.851 & \cellcolor{bestbg}\textbf{0.897} & \cellcolor{secondbg}0.745 \\
    \textbf{Qwen3-8B} & 0.624 & \cellcolor{secondbg}0.683 & \cellcolor{bestbg}\textbf{0.894} & \cellcolor{bestbg}\textbf{0.663} & \cellcolor{secondbg}0.641 & \cellcolor{bestbg}\textbf{0.874} & 0.539 & \cellcolor{secondbg}0.868 & \cellcolor{bestbg}\textbf{0.894} & \cellcolor{thirdbg}0.846 & \cellcolor{bestbg}\textbf{0.753} \\
    \bottomrule
    \end{tabular}
    }
\end{table*}

\paragraph{Baselines.}We conduct a comprehensive comparison with seven categories of baselines: 
1) \textbf{MLP};
2) \textbf{GNNs}: GCN~\cite{GCN}, GAT~\cite{GAT}, GraphSAGE~\cite{SAGE};
3) \textbf{Graph Transformers}: NodeFormer~\cite{nodeformer} and DIFFormer~\cite{DIFFormer};
4) \textbf{Small LMs}: BERT~\cite{BERT}, RoBERTa~\cite{RoBERTa}, E5~\cite{e5} and Sentence-Bert~\cite{sentence-bert};
5) \textbf{LLMs}: Qwen2-7B~\cite{qwen}, LLaMA-2-7B~\cite{llama2}, LLaMA-3.1-8B~\cite{llama3} and Vicuna-7B-v1.5~\cite{vicuna};
6) \textbf{\textit{LLM as Enhancers}}: OFA~\cite{OFA}, ZeroG~\cite{ZeroG}, GraphCLIP~\cite{GraphCLIP}, UniGLM~\cite{UniGLM} and LLM-BP~\cite{LLM-BP};
7) \textbf{\textit{LLM as Predictors}}: GraphGPT~\cite{GraphGPT}, LLaGA~\cite{LLaGA}, TEA-GLM~\cite{TEA-GLM} and GOFA~\cite{GOFA}.

\paragraph{Implementation Details.}
We use Qwen3-Embedding-8B~\cite{zhang2025qwen3emb} as the frozen unified text encoder for all datasets, and instantiate the structural extractor as a $2$-layer motif-aware GNN. As the foundation LLM, we adopt Vicuna-7B-v1.5~\cite{vicuna} as the default backbone for fair comparison with prior methods (GraphGPT, LLaGA, TEA-GLM). Generalization to other backbones, including Mistral-7B-Instruct-v0.3~\cite{mistral}, LLaMA3.1-8B-Instruct~\cite{llama3}, and Qwen3-8B-Instruct~\cite{yang2025qwen3}, is additionally reported. During alignment tuning, the LLM is frozen and only the alignment projector is optimized. All experiments are conducted with $8\times$NVIDIA H20 GPUs and CUDA $12.8$. 

\subsection{Zero-Shot In-Domain and Cross-Domain Inference Ability (RQ1)}
\label{sect:in-domain & cros-domain}

To evaluate the generalizability of our proposed GraspLLM, we conduct extensive evaluations across in-domain and cross-domain zero-shot settings.

\subsubsection{In-domain Zero-shot Ability}
We train all methods on Arxiv, Computer, Reddit, and Wisconsin, respectively, each representing a specific domain, and evaluate the zero-shot node classification performance within the same domain. 
Specifically, the WikiCS dataset is evaluated on the models trained on the Computer dataset. 
GNN-based models are trained with linear classification heads, and LLM baselines are directly applied for predictions. Additionally, we fine-tune Vicuna-7B-v1.5 with soft prompts (Vicuna-7B-v1.5-IT).

As shown in Table \ref{tab:nc_citation}, 
GraspLLM consistently outperforms baselines on most datasets, achieving superior performance.
Remarkably, it exceeds the best baselines by 0.1 on History, WikiCS, and Webpage datasets, demonstrating its effectiveness. 
In contrast to single GNNs and LLMs,
our approach integrates both to enhance zero-shot generalization. 
Prior LLM as predictor methods tend to overfit source data, which is especially serious when the distribution shift is rather small (e.g., Arxiv and Cora or Citeseer). 
In contrast, GraspLLM performs well in these scenarios, significantly mitigating this issue. 

\begin{table*}[tbp!]
    \centering
    \caption{Cross-domain zero-shot node classification accuracy of GraspLLM across four LLM backbones. For each backbone, \textbf{In-D} denotes its in-domain zero-shot accuracy on the target dataset, and \textbf{Cr-D} denotes its zero-shot accuracy when transferred from the source dataset. The delta in parentheses is Cr-D~$-$~In-D. \colorbox{ysshallowred}{Highlighted} cells with a \textcolor{red4}{red} delta mark transfers that \emph{match or surpass} the corresponding in-domain results.}
    \label{tab:nc_cross-domain}
    \scalebox{1}{
    \begin{tabular}{l|l|cc|cc|cc|cc|cc}
        \toprule
        \multirow{2}{*}{\textbf{Source}} & \multirow{2}{*}{\textbf{Target}} & \multirow{2}{*}{\textbf{LLaGA}} & \multirow{2}{*}{\textbf{TEA-GLM}} & \multicolumn{2}{c|}{\textbf{Vicuna-7B-v1.5}} & \multicolumn{2}{c|}{\textbf{Mistral-7B-v0.3}} & \multicolumn{2}{c|}{\textbf{LLaMA-3.1-8B}} & \multicolumn{2}{c}{\textbf{Qwen3-8B}} \\
        & & & & \textbf{In-D} & \textbf{Cr-D} & \textbf{In-D} & \textbf{Cr-D} & \textbf{In-D} & \textbf{Cr-D} & \textbf{In-D} & \textbf{Cr-D} \\
        \midrule
        \multirow{4}{*}{Arxiv}
        & WikiCS    & 0.589 & 0.632 & 0.855 & 0.636\,\textcolor{blue4}{\scriptsize{($-$.219)}} & 0.837 & 0.674\,\textcolor{blue4}{\scriptsize{($-$.163)}} & 0.866 & 0.728\,\textcolor{blue4}{\scriptsize{($-$.138)}} & 0.874 & 0.743\,\textcolor{blue4}{\scriptsize{($-$.131)}} \\
        & Instagram & 0.135 & 0.163 & 0.571 & 0.520\,\textcolor{blue4}{\scriptsize{($-$.051)}} & 0.636 & 0.572\,\textcolor{blue4}{\scriptsize{($-$.064)}} & 0.543 & 0.462\,\textcolor{blue4}{\scriptsize{($-$.081)}} & 0.539 & \cellcolor{ysshallowred}0.566\,\textcolor{red4}{\scriptsize{($+$.027)}} \\
        & History   & 0.375 & 0.205 & 0.656 & 0.166\,\textcolor{blue4}{\scriptsize{($-$.490)}} & 0.527 & 0.277\,\textcolor{blue4}{\scriptsize{($-$.250)}} & 0.640 & 0.636\,\textcolor{blue4}{\scriptsize{($-$.004)}} & 0.663 & 0.351\,\textcolor{blue4}{\scriptsize{($-$.312)}} \\
        & Photo     & 0.216 & 0.110 & 0.554 & 0.456\,\textcolor{blue4}{\scriptsize{($-$.098)}} & 0.627 & 0.527\,\textcolor{blue4}{\scriptsize{($-$.100)}} & 0.633 & 0.289\,\textcolor{blue4}{\scriptsize{($-$.344)}} & 0.641 & 0.435\,\textcolor{blue4}{\scriptsize{($-$.206)}} \\
        \midrule
        \multirow{3}{*}{Reddit}
        & Cora      & ---   & ---   & 0.614 & 0.388\,\textcolor{blue4}{\scriptsize{($-$.226)}} & 0.613 & 0.319\,\textcolor{blue4}{\scriptsize{($-$.294)}} & 0.635 & 0.358\,\textcolor{blue4}{\scriptsize{($-$.277)}} & 0.624 & 0.605\,\textcolor{blue4}{\scriptsize{($-$.019)}} \\
        & Pubmed    & 0.213 & 0.316 & 0.869 & 0.832\,\textcolor{blue4}{\scriptsize{($-$.037)}} & 0.819 & \cellcolor{ysshallowred}0.892\,\textcolor{red4}{\scriptsize{($+$.073)}} & 0.850 & 0.802\,\textcolor{blue4}{\scriptsize{($-$.048)}} & 0.894 & 0.881\,\textcolor{blue4}{\scriptsize{($-$.013)}} \\
        & History   & 0.199 & 0.192 & 0.656 & 0.229\,\textcolor{blue4}{\scriptsize{($-$.427)}} & 0.527 & 0.183\,\textcolor{blue4}{\scriptsize{($-$.344)}} & 0.640 & \cellcolor{ysshallowred}0.669\,\textcolor{red4}{\scriptsize{($+$.029)}} & 0.663 & 0.184\,\textcolor{blue4}{\scriptsize{($-$.479)}} \\
        \midrule
        \multirow{3}{*}{Computer}
        & Pubmed    & 0.408 & 0.496 & 0.869 & 0.750\,\textcolor{blue4}{\scriptsize{($-$.119)}} & 0.819 & \cellcolor{ysshallowred}0.881\,\textcolor{red4}{\scriptsize{($+$.062)}} & 0.850 & 0.750\,\textcolor{blue4}{\scriptsize{($-$.100)}} & 0.894 & \cellcolor{ysshallowred}0.898\,\textcolor{red4}{\scriptsize{($+$.004)}} \\
        & Photo     & ---   & ---   & 0.554 & 0.392\,\textcolor{blue4}{\scriptsize{($-$.162)}} & 0.627 & \cellcolor{ysshallowred}0.627\,\textcolor{red4}{\scriptsize{($\pm$.000)}} & 0.633 & 0.334\,\textcolor{blue4}{\scriptsize{($-$.299)}} & 0.641 & 0.408\,\textcolor{blue4}{\scriptsize{($-$.233)}} \\
        & Instagram & 0.348 & 0.240 & 0.571 & 0.368\,\textcolor{blue4}{\scriptsize{($-$.203)}} & 0.636 & 0.421\,\textcolor{blue4}{\scriptsize{($-$.215)}} & 0.543 & 0.393\,\textcolor{blue4}{\scriptsize{($-$.150)}} & 0.539 & \cellcolor{ysshallowred}0.626\,\textcolor{red4}{\scriptsize{($+$.087)}} \\
        \bottomrule
    \end{tabular}
    }
    \vspace{-1mm}
\end{table*}

\subsubsection{Cross-Domain Zero-Shot Ability}
We further assess GraspLLM's ability to generalize across domains, which presents a greater challenge than in-domain tasks. Specifically, models trained on the Arxiv, Computer, and Reddit datasets are tested on node classification tasks across datasets from different domains. 
LLaGA and TEA-GLM serve as the main baselines, which are specifically designed to integrate graph machine learning with LLMs. 

As depicted in Table \ref{tab:nc_cross-domain}, GraspLLM consistently outperforms baselines on all target datasets, with improvements exceeding 0.25 in most cases, highlighting GraspLLM's strong cross-domain generalization capability. 
It is reasonable to observe a performance drop for most transfers compared with in-domain zero-shot results, which can be attributed to semantic and distribution gaps between domains. 
Interestingly, we observe that for Pubmed, the models transferred from Reddit and Computer can match or even surpass the in-domain results across multiple backbones, and a similar phenomenon is also seen on Instagram with the Qwen3 backbone. 
This suggests that certain textual or structural properties of these datasets may align more closely with those in the source domains, facilitating better generalizability. 

\subsection{Cross-Task Zero-Shot Inference Ability (RQ2)}

\begin{table*}[h!]
    \centering
    \caption{Cross-task zero-shot link prediction AUC on ten TAG datasets, organized by domain (citation, e-commerce, weblink, and social). Per-column ranking is indicated by cell background: \colorbox{bestbg}{best}, \colorbox{secondbg}{second-best}, \colorbox{thirdbg}{third-best}.}
    \label{tab:lp-cross-task}
    \scalebox{1.0}{
    \begin{tabular}{c|cccc|ccc|c|cc}
    \toprule
    \multirow{2}{*}{\textbf{Model}} & \multicolumn{4}{c|}{\textbf{Citation}} & \multicolumn{3}{c|}{\textbf{E-commerce}} & \textbf{Weblink} & \multicolumn{2}{c}{\textbf{Social}} \\
    \cmidrule{2-11}
    & \textbf{Cora} & \textbf{Citeseer} & \textbf{Pubmed} & \textbf{Arxiv} & \textbf{History} & \textbf{Photo} & \textbf{Computer} & \textbf{WikiCS} & \textbf{Instagram} & \textbf{Reddit} \\
    \midrule
    Vicuna-7B-v1.5 & 0.518 & 0.504 & 0.535 & 0.516 & 0.523 & 0.497 & 0.512 & 0.514 & 0.501 & 0.548 \\
    Vicuna-7B-IT & 0.536 & 0.513 & 0.542 & 0.540 & 0.552 & 0.501 & 0.509 & 0.536 & 0.547 & 0.564 \\
    \midrule
    OFA & 0.486 & 0.481 & 0.492 & 0.496 & 0.441 & 0.446 & 0.468 & 0.496 & 0.487 & 0.502 \\
    ZeroG & 0.491 & 0.513 & 0.519 & 0.527 & 0.437 & 0.458 & 0.455 & 0.549 & 0.536 & 0.496 \\
    GraphCLIP & 0.527 & 0.496 & 0.501 & 0.553 & 0.499 & 0.523 & 0.499 & 0.536 & 0.598 & 0.527 \\
    \midrule
    GraphGPT & 0.531 & 0.562 & 0.510 & 0.632 & 0.510 & 0.533 & 0.524 & 0.518 & 0.613 & 0.587 \\
    LLaGA & 0.522 & 0.543 & 0.591 & 0.580 & 0.446 & 0.467 & 0.485 & 0.510 & 0.587 & 0.463 \\
    TEA-GLM & 0.586 & 0.624 & 0.689 & 0.657 & 0.579 & 0.545 & 0.554 & \cellcolor{bestbg}\textbf{0.549} & 0.632 & 0.596 \\
    \midrule
    \multicolumn{11}{l}{\cellcolor{oursbg}\textit{\textbf{GraspLLM (Ours)}}} \\
    \midrule
    \textbf{Vicuna-7B-v1.5}   & \cellcolor{thirdbg}0.642 & \cellcolor{thirdbg}0.673 & \cellcolor{thirdbg}0.728 & \cellcolor{thirdbg}0.681 & 0.587 & \cellcolor{thirdbg}0.569 & \cellcolor{thirdbg}0.557 & 0.529 & \cellcolor{thirdbg}0.648 & 0.619 \\
    \textbf{Mistral-7B-v0.3}  & 0.629 & 0.658 & 0.711 & 0.665 & \cellcolor{bestbg}\textbf{0.611} & \cellcolor{secondbg}0.582 & \cellcolor{secondbg}0.566 & \cellcolor{thirdbg}0.544 & \cellcolor{bestbg}\textbf{0.674} & \cellcolor{secondbg}0.638 \\
    \textbf{LLaMA-3.1-8B} & \cellcolor{bestbg}\textbf{0.661} & \cellcolor{secondbg}0.689 & \cellcolor{secondbg}0.736 & \cellcolor{bestbg}\textbf{0.704} & \cellcolor{thirdbg}0.598 & 0.567 & 0.553 & \cellcolor{secondbg}0.547 & 0.633 & \cellcolor{thirdbg}0.628 \\
    \textbf{Qwen3-8B}    & \cellcolor{secondbg}0.658 & \cellcolor{bestbg}\textbf{0.694} & \cellcolor{bestbg}\textbf{0.751} & \cellcolor{secondbg}0.692 & \cellcolor{secondbg}0.605 & \cellcolor{bestbg}\textbf{0.585} & \cellcolor{bestbg}\textbf{0.572} & 0.541 & \cellcolor{secondbg}0.659 & \cellcolor{bestbg}\textbf{0.642} \\
    \bottomrule
    \end{tabular}
    }
\end{table*}

We employ models trained on node classification tasks directly for link prediction tasks without additional fine-tuning, evaluating the cross-task generalizability. 
Compared to node classification, link prediction differs in both granularity and objective.
The AUC scores are reported in Table~\ref{tab:lp-cross-task} across LLM-only, LLM as Enhancer, and LLM as Predictor baselines.
Fine-tuning improves Vicuna-7B-IT over its base version, yet both remain suboptimal due to insufficient structural understanding. Previous LLM as predictor methods, such as GraphGPT and LLaGA, exhibit poor generalization, largely due to overfitting to task-specific training data. 
TEA-GLM performs better but is still constrained by its narrow focus on local structural patterns and the limited generalization of the GNN. 
In contrast, GraspLLM consistently outperforms all baselines across most datasets, and this advantage holds across all four LLM backbones, with Qwen3 and LLaMA-3.1 ranking best on the majority of columns. 
The robustness of the gains across backbones suggests that GraspLLM captures task-agnostic structural priors and thus mitigates the risk of overfitting during training.

\begin{figure*}[h!]
    \centering
    \includegraphics[width=0.95\textwidth]{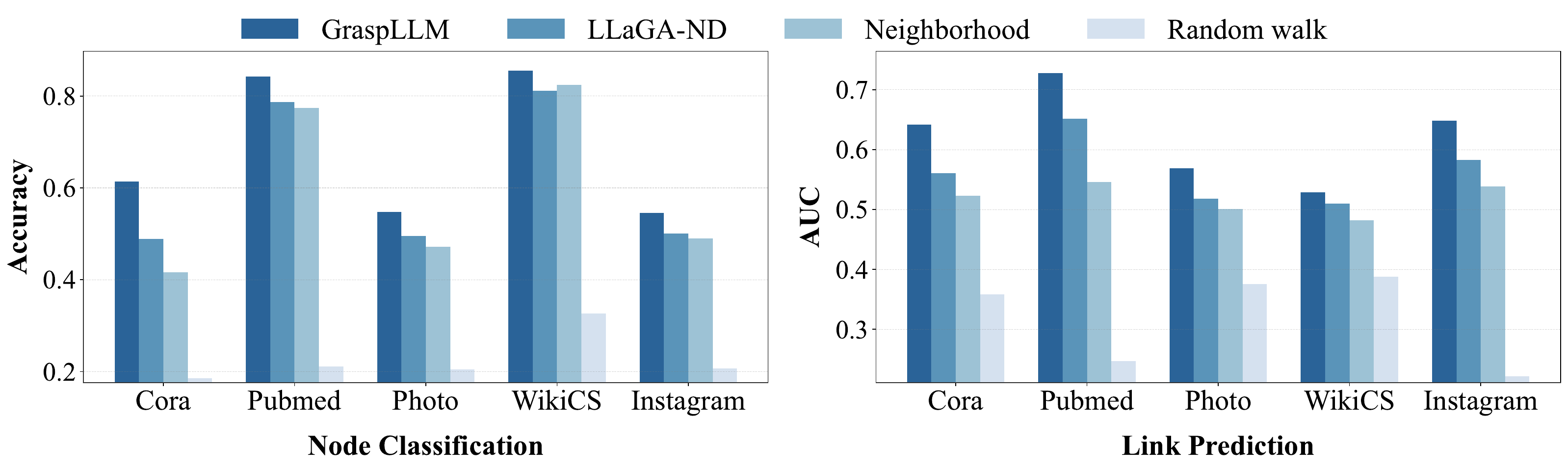}
    \caption{Comparison of four subgraph sampling and structure organizing strategies: LLaGA-ND, $2$-hop Neighborhood, Random Walk, and our \textit{Optimal Contextual Subgraph}. Each group of bars contrasts the four strategies on one dataset.}
    \label{fig:subgraph}
    \vspace{-1mm}
\end{figure*}

\subsection{Analysis on \textit{Optimal Contextual Subgraph Sampling} (RQ3)}
\label{sect:subgraph-sampling}

We evaluate the effectiveness of \textbf{\textit{optimal contextual subgraph sampling}} strategy (Section~\ref{sect:subgraph sampling}) against three 
baseline methods, including the Neighborhood Detail Template from LLaGA (LLaGA-ND), 1-hop and 2-hop neighbor sampling (Neighborhood), and vanilla random walk (Random Walk). 

\begin{table*}[h]
    \centering
    \caption{Ablation study of the two key components in GraspLLM. Each row corresponds to a variant of the full framework: ``w/o Structural Extractor'' isolates the contribution of the motif-aware self-supervised GNN by replacing it with a randomly initialized one, while ``w/o Alignment Projector'' isolates the contribution of our projector by replacing it with a plain linear projector.}
    \label{tab:ablations}
    \resizebox{\textwidth}{!}{
    \begin{tabular}{c|ccccc|ccccc}
    \toprule
    \multirow{2}{*}{\textbf{Method}} & \multicolumn{5}{c|}{\textbf{Node Classification}} & \multicolumn{5}{c}{\textbf{Link Prediction}} \\
    \cmidrule{2-11}
    & \textbf{Cora} & \textbf{Pubmed} & \textbf{Photo} & \textbf{WikiCS} & \textbf{Instagram} & \textbf{Cora} & \textbf{Pubmed} & \textbf{Photo} & \textbf{WikiCS} & \textbf{Instagram} \\
    \midrule
        w/o Structural Extractor & 0.474 & 0.798 & 0.491 & 0.713 & 0.448 & 0.587 & 0.679 & 0.483 & 0.481 & 0.518 \\
    w/o Alignment Projector  & 0.509 & 0.835 & 0.519 & 0.794 & 0.531 & 0.611 & 0.701 & 0.527 & 0.504 & 0.602 \\ \midrule
    \textbf{GraspLLM}                 & \textbf{0.614} & \textbf{0.869} & \textbf{0.554} & \textbf{0.855} & \textbf{0.571} & \textbf{0.642} & \textbf{0.728} & \textbf{0.569} & \textbf{0.529} & \textbf{0.648} \\
    \bottomrule
    \end{tabular}
    }
    \vspace{-1mm}
\end{table*}

\begin{table*}[h]
    \centering
    \caption{Supervised node classification accuracy on ten TAG benchmarks, organized by domain (citation, e-commerce, weblink, and social). The rightmost \textbf{AVG} column averages each method's accuracy over the ten datasets. Per-column ranking is indicated by cell background: \colorbox{bestbg}{best}, \colorbox{secondbg}{second-best}, \colorbox{thirdbg}{third-best}.}
    \label{tab:nc-supervised-little}
    \resizebox{\textwidth}{!}{
        \begin{tabular}{c|cccc|ccc|c|cc|c}
        \toprule
        \multirow{2}{*}{\textbf{Model}} & \multicolumn{4}{c|}{\textbf{Citation}} & \multicolumn{3}{c|}{\textbf{E-commerce}} & \textbf{Weblink} & \multicolumn{2}{c|}{\textbf{Social}} & \multirow{2}{*}{\textbf{AVG}} \\
        \cmidrule{2-11}
        & \textbf{Cora} & \textbf{Citeseer} & \textbf{Pubmed} & \textbf{Arxiv} & \textbf{History} & \textbf{Photo} & \textbf{Computer} & \textbf{WikiCS} & \textbf{Instagram} & \textbf{Reddit} & \\
        \midrule
        GCN       & 0.854 & 0.745 & 0.886 & 0.727 & 0.746 & 0.764 & 0.809 & 0.843 & 0.649 & 0.645 & 0.767 \\
        GAT       & 0.877 & 0.746 & 0.887 & 0.720 & 0.810 & 0.844 & 0.885 & 0.825 & 0.625 & 0.600 & 0.782 \\
        GraphSAGE & 0.825 & 0.691 & 0.874 & 0.713 & 0.790 & 0.850 & 0.861 & 0.839 & 0.631 & 0.617 & 0.769 \\
        \midrule
        Vicuna-7B-v1.5 & 0.705 & 0.639 & 0.932 & 0.711 & 0.847 & 0.787 & 0.745 & 0.802 & 0.476 & 0.523 & 0.717 \\
        \midrule
        ENGINE    & 0.874 & 0.735 & 0.910 & 0.747 & 0.851 & 0.859 & 0.892 & \cellcolor{secondbg}0.864 & \cellcolor{secondbg}0.679 & 0.656 & 0.807 \\
        TAPE      & 0.880 & 0.761 & 0.932 & 0.739 & 0.841 & 0.860 & 0.901 & 0.831 & 0.632 & 0.651 & 0.803 \\
        GraphCLIP & 0.647 & 0.659 & 0.577 & 0.532 & 0.739 & 0.735 & 0.782 & 0.795 & 0.556 & 0.606 & 0.663 \\
        \midrule
        GraphGPT  & 0.865 & 0.757 & 0.936 & 0.751 & 0.854 & 0.835 & 0.879 & 0.834 & 0.627 & 0.640 & 0.798 \\
        LLaGA     & 0.892 & 0.788 & \cellcolor{thirdbg}0.951 & \cellcolor{secondbg}0.767 & 0.867 & 0.863 & 0.902 & 0.851 & \cellcolor{bestbg}\textbf{0.684} & 0.662 & 0.823 \\
        TEA-GLM   & 0.873 & 0.740 & 0.866 & 0.655 & 0.839 & 0.716 & 0.578 & 0.796 & 0.595 & 0.547 & 0.721 \\
        \midrule
        \multicolumn{12}{l}{\cellcolor{oursbg}\textit{\textbf{GraspLLM (Ours)}}} \\
        \midrule
        \textbf{Vicuna-7B-v1.5}   & \cellcolor{thirdbg}0.900 & \cellcolor{bestbg}\textbf{0.827} & \cellcolor{secondbg}0.953 & 0.761 & \cellcolor{thirdbg}0.882 & 0.879 & \cellcolor{thirdbg}0.916 & 0.857 & 0.660 & \cellcolor{secondbg}0.672 & \cellcolor{thirdbg}0.831 \\
        \textbf{Mistral-7B-v0.3}  & 0.895 & 0.802 & 0.944 & 0.752 & 0.873 & \cellcolor{secondbg}0.882 & 0.912 & 0.849 & \cellcolor{thirdbg}0.674 & 0.665 & 0.825 \\
        \textbf{LLaMA-3.1-8B} & \cellcolor{secondbg}0.906 & \cellcolor{secondbg}0.823 & \cellcolor{thirdbg}0.951 & \cellcolor{bestbg}\textbf{0.770} & \cellcolor{bestbg}\textbf{0.890} & \cellcolor{thirdbg}0.880 & \cellcolor{secondbg}0.919 & \cellcolor{bestbg}\textbf{0.868} & 0.652 & \cellcolor{thirdbg}0.668 & \cellcolor{secondbg}0.833 \\
        \textbf{Qwen3-8B}    & \cellcolor{bestbg}\textbf{0.911} & \cellcolor{thirdbg}0.819 & \cellcolor{bestbg}\textbf{0.964} & \cellcolor{thirdbg}0.763 & \cellcolor{secondbg}0.886 & \cellcolor{bestbg}\textbf{0.889} & \cellcolor{bestbg}\textbf{0.928} & \cellcolor{thirdbg}0.860 & 0.657 & \cellcolor{bestbg}\textbf{0.677} & \cellcolor{bestbg}\textbf{0.835} \\
        \bottomrule
        \end{tabular}
    }
    \vspace{-1mm}
\end{table*}

As shown in Fig.~\ref{fig:subgraph},  our method consistently outperforms the other three approaches across all datasets. LLaGA-ND ranks second, followed by the neighbor-based method, both focusing on local neighborhoods. In contrast, our approach achieves better performance by exploring the graph more deeply and capturing richer structural insights. Despite delving deeper into the graph, Random Walk performs significantly worse. This can be attributed to the randomly sampled subgraphs lacking coherent contextual and structural cues, leading to confusion for the LLM. 
In contrast, our \textbf{\textit{optimal contextual subgraph}} captures both semantic and structural information, empowering LLMs to acquire more generalizable graph representations.

\paragraph{Case Study.}
\begin{figure}[t]
    \centering
    \begin{subfigure}{\columnwidth}
        \centering
        \includegraphics[width=\linewidth]{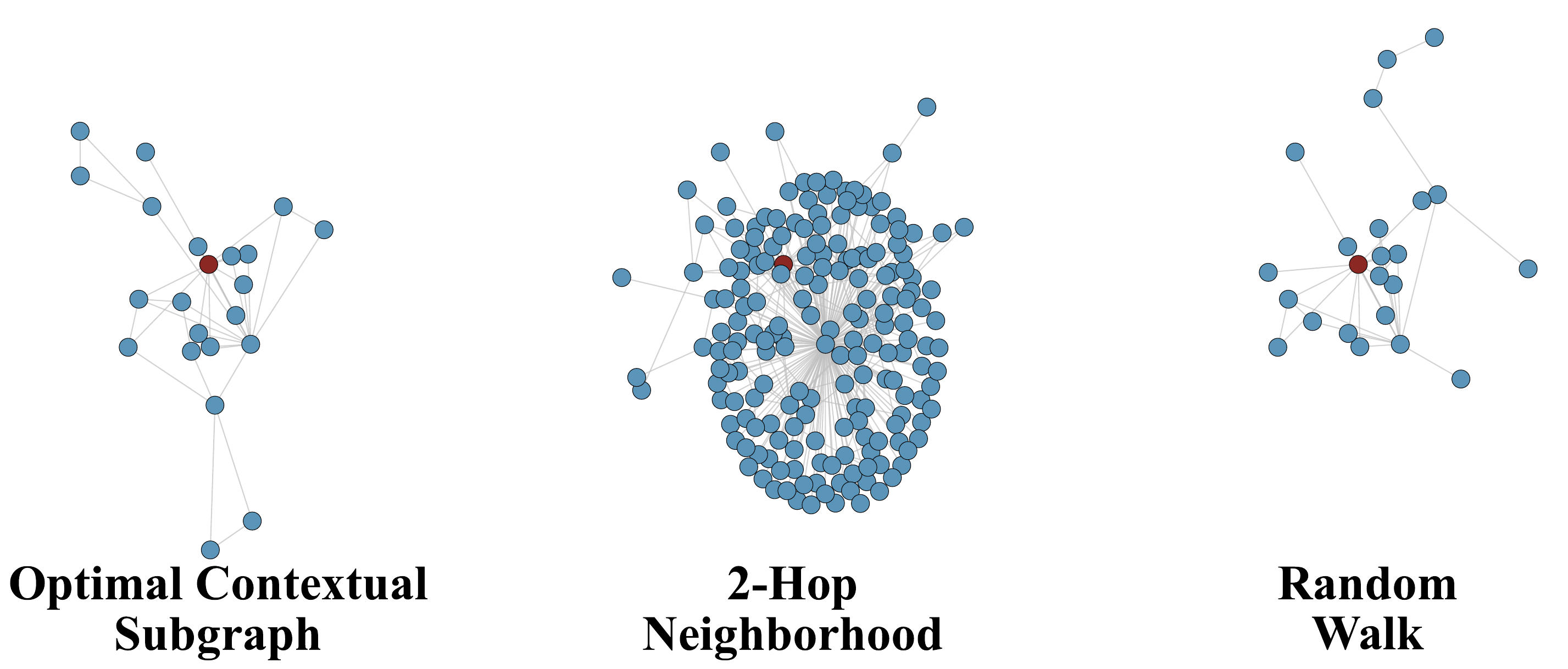}
        \caption{Cora (center node $\#1153$)}
        \label{fig:case-subgraph-cora}
    \end{subfigure}
    \vspace{6pt}
    \begin{subfigure}{\columnwidth}
        \centering
        \includegraphics[width=\linewidth]{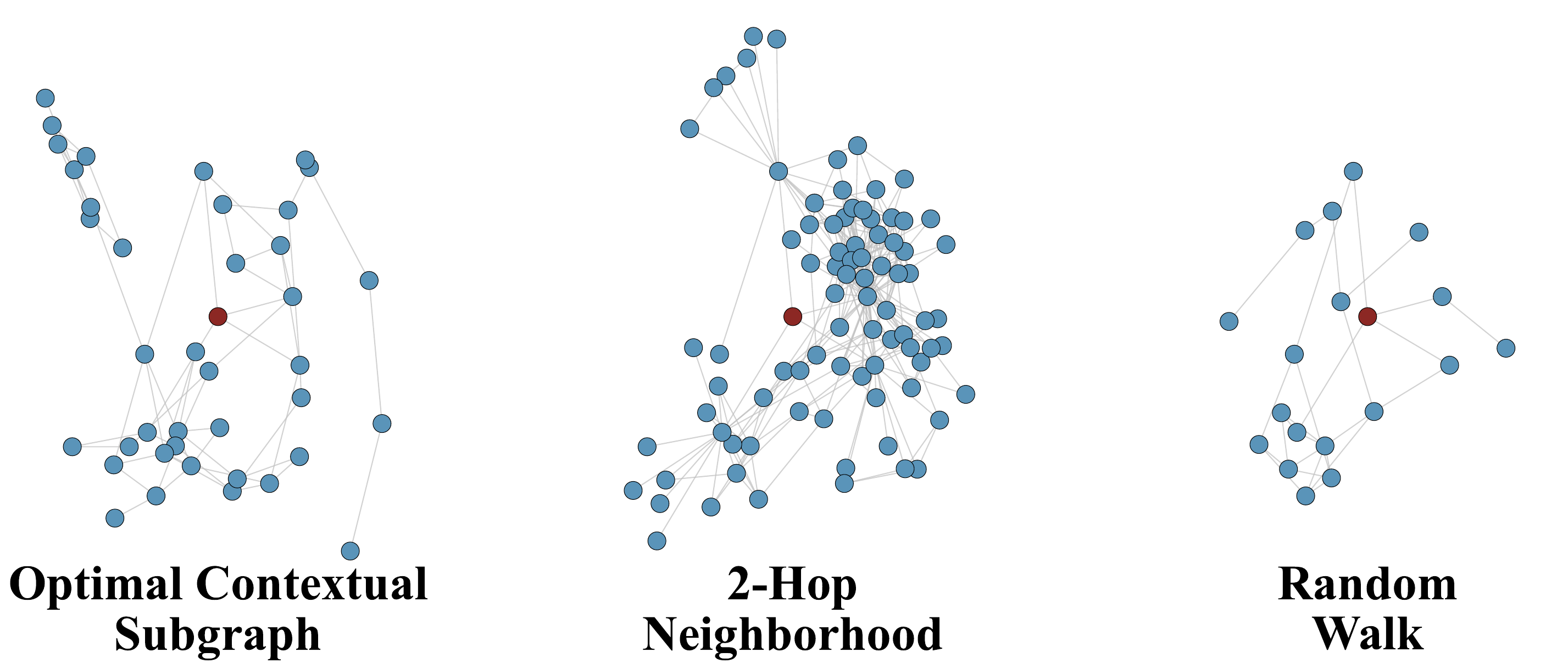}
        \caption{Pubmed (center node $\#1540$)}
        \label{fig:case-subgraph-pubmed}
    \end{subfigure}
    \vspace{6pt}
    \begin{subfigure}{\columnwidth}
        \centering
        \includegraphics[width=\linewidth]{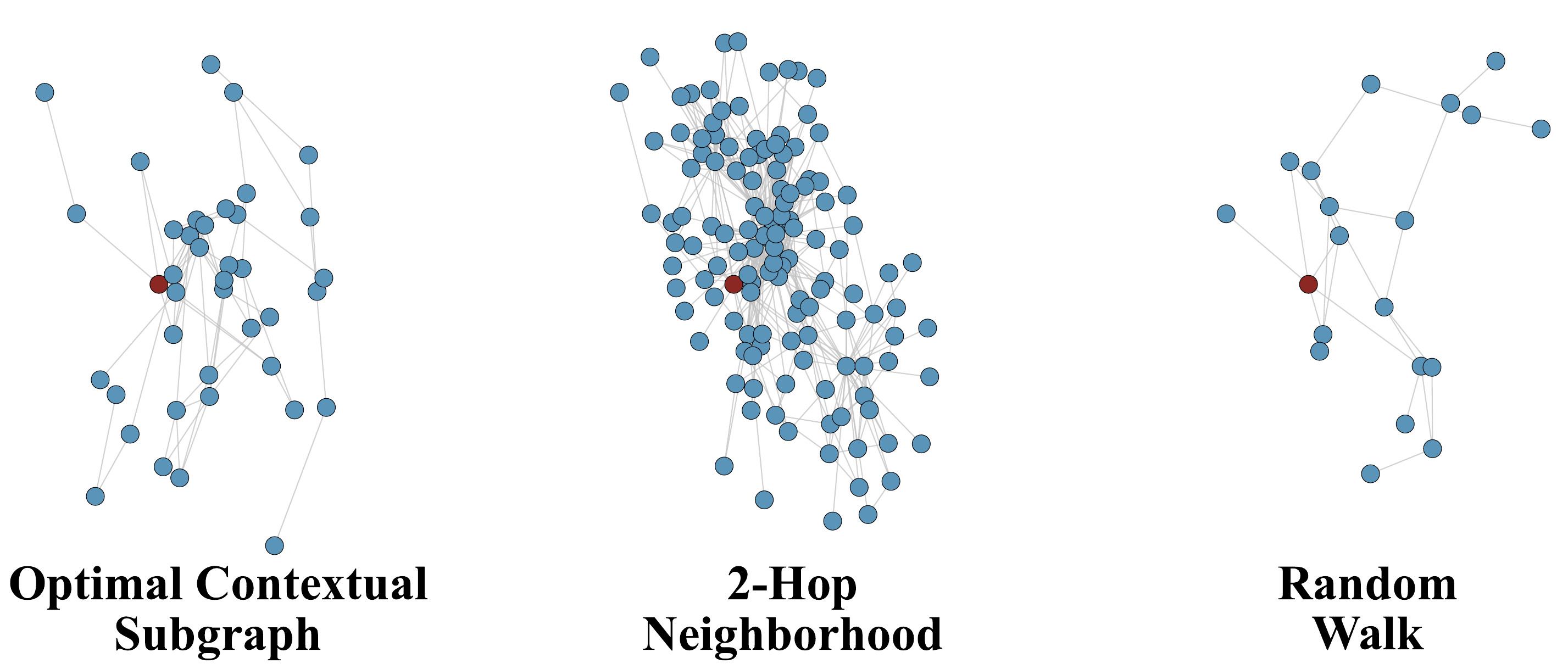}
        \caption{Instagram (center node $\#2877$)}
        \label{fig:case-subgraph-instagram}
    \end{subfigure}
    \caption{Visualization of subgraphs sampled by the Neighborhood, Random Walk, and our Optimal Contextual Subgraph strategies on three representative datasets (center nodes marked in \textcolor[HTML]{8C2824}{red}, and the sampled nodes are marked in \textcolor[HTML]{5A94B9}{blue}).}
    \label{fig:case-subgraph}
    \vspace{-1mm}
\end{figure}
We further visualize three representative cases spanning citation (Cora), biomedical (Pubmed), and social (Instagram) graphs (Fig.~\ref{fig:case-subgraph}) to make the qualitative differences concrete. Across all three datasets, the neighborhood sampler is dominated by local connections, which dilutes the center node with neighboring clusters and risks exponential expansion. Random Walk explores deeper but produces nodes that are mutually weakly related, offering limited contextual coherence. 
In contrast, our \textit{Optimal Contextual Subgraph} consistently yields a compact yet far-reaching subgraph that filters out marginally relevant nodes and captures informative nodes from as far as $10$ hops away, demonstrating its robustness across diverse graph domains.

\paragraph{Parameter Sensitivity.}
\begin{figure}[t]
    \centering
    \includegraphics[width=1.0\columnwidth]{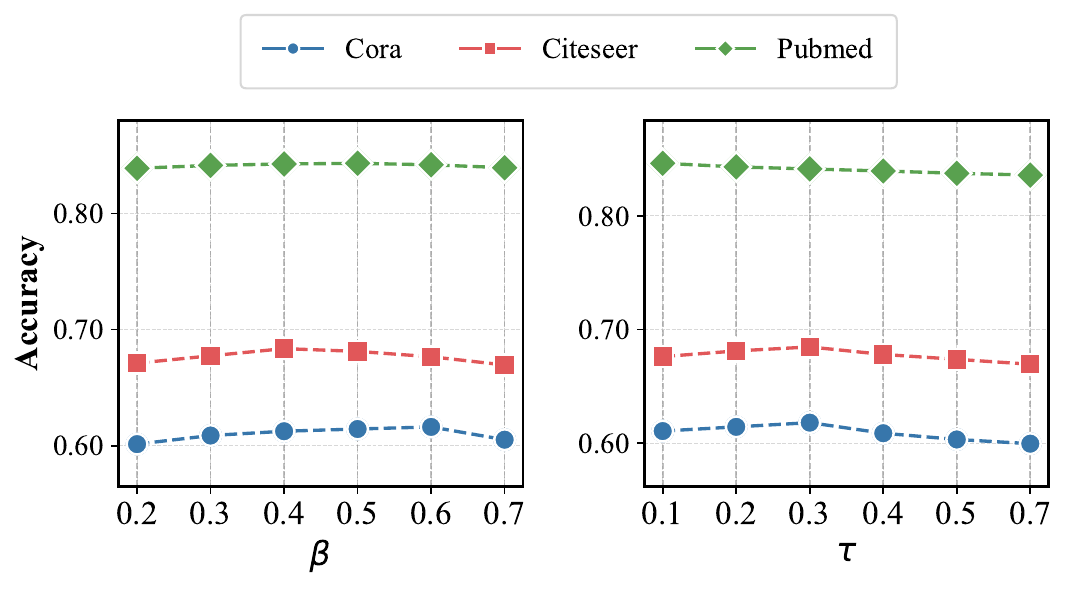}
    \caption{Sensitivity of GraspLLM to the trade-off coefficient $\beta$ and the stopping threshold $\tau$, evaluated on three citation datasets under zero-shot node classification.}
    \label{fig:param}
    \vspace{-1mm}
\end{figure}
We further study how GraspLLM responds to the two parameters introduced in subgraph sampling: $\beta$, which balances structural coherence against center relevance, and $\tau$, which determines when the greedy expansion terminates (Fig.~\ref{fig:param}). 
Varying $\beta$ in $[0.2, 0.7]$, performance is stable across the studied range and peaks at moderate values around $\beta\!\in\![0.4, 0.6]$, while it slightly drops at both endpoints, confirming that relying excessively on either signal yields suboptimal subgraphs. 
For $\tau$, larger values restrict expansion to a few highly relevant neighbors and discard long-range structural cues, leading to a clear performance decline at $\tau\!=\!0.7$. Meanwhile, smaller thresholds $\tau$ consistently work best across all three datasets.

\subsection{Ablation Study (RQ4)}

To validate the key components of our framework, we conduct ablation studies by masking them individually. 
The variant ``w/o Structural Extractor'' removes the structural information extractor (Section~\ref{sect:structural information}) and uses a randomly initialized GNN; and ``w/o Alignment Projector'' replaces our alignment projector (Section~\ref{sect:alignment tuning}) with a simple linear projector. 
The results in Table~\ref{tab:ablations} demonstrate the effectiveness of each component. 
The most significant performance decline arises from ``w/o Structural Extractor'', with consistent drops across all datasets and reductions of up to 0.14 on node classification (e.g., WikiCS) and 0.13 on link prediction (e.g., Instagram).
This highlights that the node representations generated by the GNN contain rich structural information, verifying the domain-invariant structural information extraction ability of the model. 
The ``w/o Alignment Projector'' also causes a clear decline across all datasets, indicating the critical role of our alignment strategy in enabling the model to learn more generalized feature and structural knowledge.
Therefore, GraspLLM indeed benefits from these two components.

\subsection{Supervised Results}

In this section, we investigate the supervised performance of GraspLLM (Table~\ref{tab:nc-supervised-little}). For most datasets, we follow the same configurations as in the zero-shot instruction tuning phase. For datasets with limited size such as Cora and Citeseer, all methods are tuned for $6$ epochs while other settings remain unchanged.
GraphGPT and LLaGA demonstrate strong performance in supervised scenarios but degrade significantly in zero-shot scenarios. Conversely, TEA-GLM achieves impressive zero-shot results but falls short in supervised tasks. 
This indicates a common dilemma in prior methods: Improved zero-shot generalization often comes with diminished supervised performance. 
However, GraspLLM excels in both settings, achieving state-of-the-art supervised results while retaining strong zero-shot capabilities. This demonstrates its superior adaptability, a dual strength not observed in previous approaches.

\subsection{Robustness to Prompt Format}
\label{sect:prompt-robustness}

Since GraspLLM is trained with a fixed instruction format (Format~1, Fig.~\ref{prompt:node_classification}), a natural concern is whether its zero-shot ability is brittle to prompt rewording at inference time. To probe this, we evaluate the trained model with two alternative prompts \textit{without any further tuning}: \textit{Format~2} replaces the structural description with a more concise wording (``\textit{Given a graph $\langle$graph$\rangle$, where each node contains its textual attribute$\ldots$}''), and \textit{Format~3} adopts a markedly different narrative style (``\textit{Consider a scenario where a graph is formed by multiple nodes$\ldots$ examine the main node$\ldots$}''); Format~3 thus departs furthest from the training format.

\begin{table}[h]
    \centering
    \caption{Zero-shot node classification accuracy under different prompt formats at inference time (no re-training). LLaGA and TEA-GLM are listed as references with their respective original prompts.}
    \label{tab:prompt_robustness_results}
    \scalebox{1.0}{
    \begin{tabular}{c|ccccc}
    \toprule
    \textbf{Method} & \textbf{Cora} & \textbf{Pubmed} & \textbf{Photo} & \textbf{WikiCS} & \textbf{Instagram} \\
    \midrule
    LLaGA   & 0.348 & 0.793 & 0.392 & 0.710 & 0.479 \\
    TEA-GLM & 0.379 & 0.848 & 0.497 & 0.691 & 0.508 \\
    \midrule
    \multicolumn{6}{l}{\cellcolor{oursbg}\textit{\textbf{GraspLLM (Ours)}}} \\
    \midrule
    \textbf{Format~1} & 0.614 & 0.869 & 0.554 & 0.855 & 0.571 \\
    \textbf{Format~2} & 0.589 & 0.774 & 0.529 & 0.813 & 0.523 \\
    \textbf{Format~3} & 0.557 & 0.752 & 0.408 & 0.769 & 0.466 \\
    \bottomrule
    \end{tabular}
    }
    \vspace{-1mm}
\end{table}

As shown in Table~\ref{tab:prompt_robustness_results}, accuracy unsurprisingly drops as the inference-time format diverges from the training one~\cite{Liang2023ExploringFC}, with the largest gap appearing under the most distinct Format~3. Even so, GraspLLM under the most adversarial Format~3 remains competitive with the in-domain LLaGA and TEA-GLM baselines (evaluated with their own original prompts), and on Cora and WikiCS it still leads by a clear margin, indicating that the gains brought by our structural extractor and contextual subgraph are not an artefact of any particular prompt template.

\subsection{Efficiency and Scalability Analysis}
\label{sect:efficiency-scalability}

\paragraph{Efficiency.}
\begin{figure}[t]
    \centering
    \includegraphics[width=1.0\columnwidth]{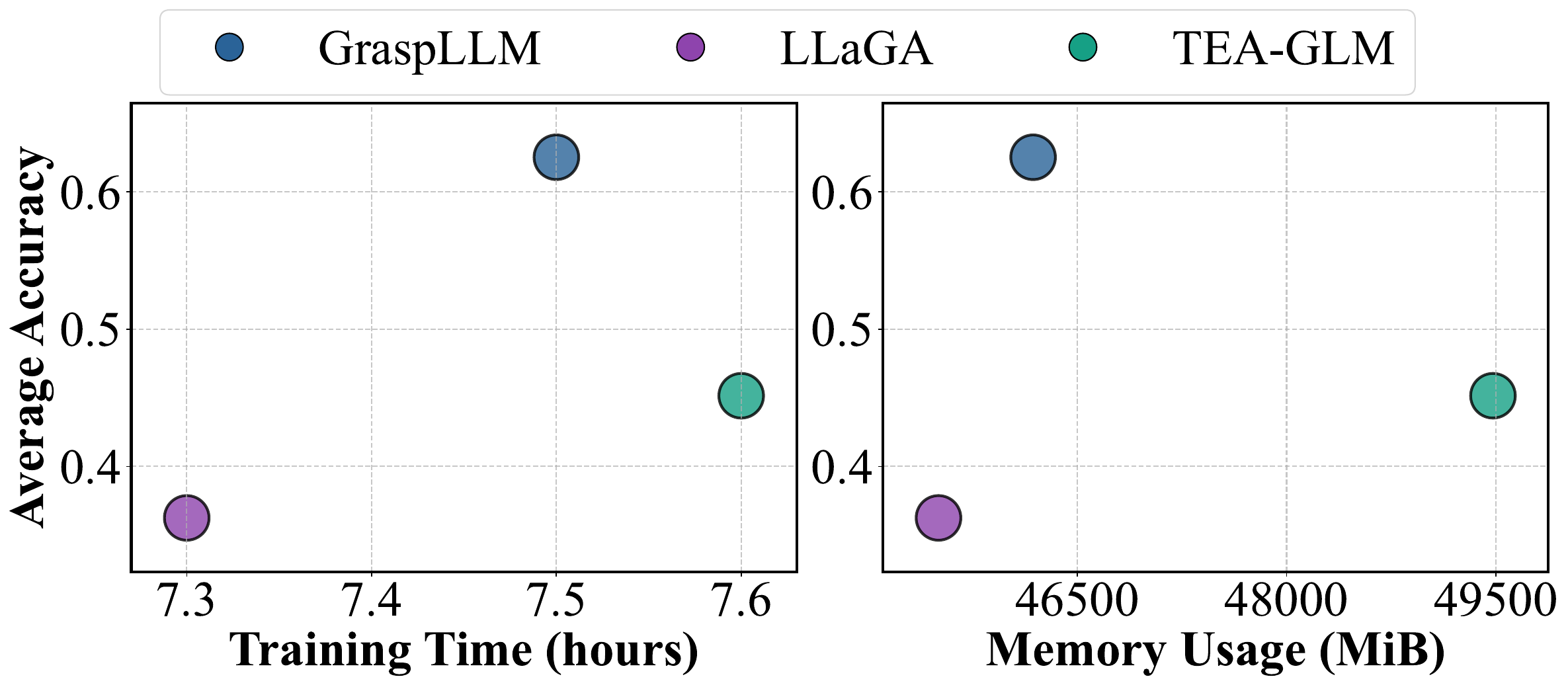}
    \caption{Efficiency analysis of GraspLLM. Each point is labeled with the average accuracy of the zero-shot node classification task across Cora, Citeseer, History and Photo.}
    \label{fig:efficient}
    \vspace{-1mm}
\end{figure}
In Fig.~\ref{fig:efficient}, we analyze the efficiency of GraspLLM, including the training time and memory consumption on the Arxiv dataset together with the zero-shot node classification performance. We observe that the time and memory overhead of our framework is comparable to those of the baseline methods (LLaGA and TEA-GLM), while achieving superior performance.

\paragraph{Scalability.}

\begin{table}[t]
    \centering
    \caption{End-to-end cost of GraspLLM on a large TAG (OGBN-Products~\cite{OGB}, $\sim$2.45M nodes / $\sim$62M edges) under single-GPU and 8-GPU data-parallel settings. The 8-GPU subgraph-sampling row corresponds to the best configuration in Table~\ref{tab:engineering-ablation}, and $\dagger$~denotes \textit{owner / peer} GPU memory under CUDA-IPC shared embedding.}
    \label{tab:scalability}
    \resizebox{\columnwidth}{!}{
    \begin{tabular}{lcc}
    \toprule
    \textbf{Action} & \textbf{Time} & \textbf{Memory$\dagger$} \\
    \midrule
    \multicolumn{3}{l}{\cellcolor{oursbg}{\textbf{\textit{Single-GPU}}}} \\
    \midrule
    GNN inference     & 0.91 s (100{,}000 nodes)   & 2{,}429 MiB  \\
    Subgraph sampling & 23.0 min (100{,}000 nodes) & 25.2 GB \\
    LLM tuning        & 14.1 h                     & 47.4 GB \\
    LLM inference     & 112 min (10{,}000 nodes)   & 20.1 GB \\
    \midrule
    \multicolumn{3}{l}{\cellcolor{oursbg}{\textbf{\textit{8-GPU}}}} \\
    \midrule
    Subgraph sampling & 7.8 min (100{,}000 nodes)  & 33.9 / 9.8 GB$^{\dagger}$ \\
    LLM tuning        & 2.4 h                      & 49.0 GB / GPU \\
    \bottomrule
    \end{tabular}
    }
    \vspace{-1mm}
\end{table}

To assess the scalability of GraspLLM on large TAGs, we evaluate on OGBN-Products~\cite{OGB}, the largest publicly available TAG benchmark, with approximately $2.45$M nodes and $62$M edges.
As shown in Table~\ref{tab:scalability}, despite the order-of-magnitude increase in graph size relative to the fourteen benchmarks, the end-to-end time and memory cost of GraspLLM remain moderate and decrease substantially under $8$-GPU data parallelism, demonstrating that the framework is deployable beyond the standard medium-scale TAG regime.

To attribute these gains to the engineering layers introduced in Section~\ref{sect:engineering}, Table~\ref{tab:engineering-ablation} progressively enables them and reports end-to-end subgraph-sampling time on the same graph.
Starting from a CPU-side serial reference whose cost is essentially prohibitive at this scale, the GPU-CSR adjacency with batched $9$-walk kernel already brings single-GPU sampling into the practical regime ($23$ min for $100{,}000$ centers); sharding centers across $8$ GPUs more than halves the wall-clock time; the CUDA-IPC shared embedding additionally lowers the per-GPU memory of peer workers from ${\sim}26$\,GB to ${\sim}9$\,GB by removing $7\times$ redundant fp16 embedding replicas; and finally multi-center batching pushes the end-to-end time to ${\sim}7.8$ min for $100{,}000$ centers.
We emphasize that none of these layers modifies the sampling algorithm itself (Section~\ref{sect:engineering}); they merely re-express the same computation in a more GPU-friendly form, so all accuracy results in Sections~\ref{sect:in-domain & cros-domain}--\ref{sect:prompt-robustness} carry over unchanged. Moreover, our framework is orthogonal to the choice of GNN and can readily incorporate scalable GNN variants for graphs beyond the current scale.

\section{Discussions}

\subsection{Challenges and Limitations of GraspLLM}

\paragraph{Challenges}
Across the cross-domain evaluation in Table~\ref{tab:nc_cross-domain}, the residual transfer gap is highly uneven: pairs whose source and target graphs share similar topological statistics (e.g., dense, low-clustering citation graphs) transfer nearly losslessly, while pairs that bridge structurally dissimilar (e.g., from a social graph to a citation graph) can still drop substantially on individual backbones. This suggests that the gap is governed by structural distribution mismatch rather than by raw graph size or label. 
A related challenge surfaces in Table~\ref{tab:nc-supervised-little} on graphs whose neighborhood semantics correlate weakly with node labels, such as the binary-class Instagram network: in these regimes the contextual subgraph contributes less information than in label-aligned citation graphs, and the supervised margin of GraspLLM over strong baselines narrows accordingly.

\paragraph{Limitations}
Two limitations are intrinsic to the current scope of GraspLLM.
First, our study focuses exclusively on Text-Attributed Graphs and does not yet cover other graph families pervasive in real data systems, such as heterogeneous, signed, or temporally evolving graphs whose nodes and edges carry richer typing or time information beyond a single textual attribute.
Second, on the task side we restrict the evaluation to node- and edge-level zero-shot inference (node classification and link prediction), and have not assessed GraspLLM on broader graph-analytical tasks such as graph-level classification, community detection, or other cluster-level workloads.

\subsection{Opportunities for Improvement}

Beyond the limitations above, we highlight two broader opportunities that we believe are particularly worth pursuing.

\textit{\textbf{More principled graph tokenization for LLMs.}} GraspLLM currently feeds the LLM with a sampled \emph{Optimal Contextual Subgraph} that is then linearized into a token sequence through a cross-attention pooling projector. 
While this design is effective on the benchmarks studied, some information loss in the linearization step is unavoidable, and a growing body of recent work has begun to investigate this issue from different angles~\cite{yehudai2026depth, bechler2026lost}. 
A promising direction is therefore to design tokenization schemes that, in addition to compressing the subgraph, are explicitly aware of these trade-offs, so that the LLM receives a representation that is both sequence-compatible and provably less lossy than a single fixed expansion.

\textit{\textbf{From fixed graph inputs to LLM-explorable graph trajectories.}} Our current pipeline treats the graph as a static input that is sampled, tokenized, and consumed in a single LLM forward pass. As LLM base models grow stronger and agentic frameworks mature, an alternative is to view graph reasoning from the \emph{data} side: preprocess the raw graph into trajectory-style data (e.g., goal-conditioned walks, intermediate query plans) that better matches how an agentic LLM explores knowledge step by step, rather than relying on a single fixed context window. 
This turns the graph store into an active provider of LLM-friendly exploration traces, and opens a complementary path to scaling LLM-based graph analytics.

\begin{table}[t]
    \centering
    \caption{Per-layer contribution of the engineering optimizations in Section~\ref{sect:engineering}, measured end-to-end on OGBN-Products. Each row is cumulative w.r.t.~the row above. $\dagger$~For rows with CUDA-IPC, memory is reported as \textit{owner / peer} GPU.}
    \label{tab:engineering-ablation}
    \resizebox{\columnwidth}{!}{
    \begin{tabular}{lcc}
    \toprule
    \textbf{Configuration} & \textbf{Time (100k)} & \textbf{Peak Mem.$^{\dagger}$} \\
    \midrule
    Reference (CPU, serial)                & $\gg 100$\,h & --- \\
    + GPU-CSR \& batched $9$-walk kernel   & 23.0 min     & 25.2 GB \\
    + $8$-GPU center sharding              & 10.2 min     & 25.5 GB / GPU \\
    + CUDA-IPC shared embedding            & \phantom{0}9.4 min  & 29.8 / 8.8 GB \\
    + Multi-center batching ($K{=}4$)      & \textbf{\phantom{0}7.8 min} & 33.9 / 9.8 GB \\
    \bottomrule
    \end{tabular}
    }
    \vspace{-1mm}
\end{table}

\section{Conclusion}

This paper presents GraspLLM, a backbone-agnostic framework that strengthens the zero-shot generalizability of LLMs on Text-Attributed Graphs by unifying node features in a shared semantic space, distilling domain-agnostic structural priors via motif-aware self-supervised learning, and constructing an \textit{Optimal Contextual Subgraph} that is aligned to the token space of a frozen LLM. Extensive experiments on fourteen TAG benchmarks across five domains and four LLM backbones show that GraspLLM consistently achieves state-of-the-art performance under in-domain, cross-domain, and cross-task zero-shot settings while remaining competitive in supervised settings and scalable to million-node graphs.

\section*{Acknowledgment}

This work is supported by Fundamental and Interdisciplinary Disciplines Breakthrough Plan of the Ministry of Education of China (JYB2025XDXM113), National Natural Science Foundation of China (92470121, 62402016), National Key R\&D Program of China (2024YFA1014003), Zhongguancun Academy (C20250204, C20250602),  Beijing Major Science and Technology Project (Z251100008125043, Z251100008425023), and High-performance Computing Platform of Peking University.

\bibliography{citations}
\bibliographystyle{IEEEtran}

\end{document}